\newif\ifcomments 
\commentstrue
\documentclass{article}

\usepackage{amsmath,amssymb,pifont}
\usepackage{multicol}
\usepackage{amstext}
\usepackage{amsthm}
\usepackage{multirow}
\usepackage{booktabs}
\usepackage[skip=0pt]{subcaption}
\usepackage{times}
\usepackage{lipsum}
\usepackage[shortlabels]{enumitem}
\usepackage{cancel}
\usepackage{wrapfig}
\usepackage{array}
\usepackage{siunitx}
\usepackage{csvsimple}
\usepackage[multidot]{grffile}
\usepackage{bbm}
\usepackage{dblfloatfix}
\usepackage{hyperref}
\usepackage{makecell}
\usepackage{bbm, dsfont}
\usepackage{mathtools}
\usepackage{xcolor}
\usepackage{comment}

\usepackage[multiple]{footmisc}
\usepackage{mathrsfs}
\usepackage{tikz}
\usepackage{hyperref}
\usepackage{cleveref}

\usepackage{algpseudocode,algorithm,algorithmicx}
\usepackage{xfrac}

\newcommand{\eanoteinline}[1]{\todo[color=red!30,inline]{EA: #1}}

\hypersetup{final}

\newcommand{\boldzero}{\ensuremath{\boldsymbol{0}}}

\newcommand{\bfH}{\ensuremath{\mathbf{H}}}

\newcommand{\bfX}{\ensuremath{\mathbf{X}}}

\newcommand{\bfb}{\ensuremath{\mathbf{b}}}

\newcommand{\bfg}{\ensuremath{\mathbf{g}}}

\newcommand{\bfv}{\ensuremath{\mathbf{v}}}

\newcommand{\bfx}{\ensuremath{\mathbf{x}}}
\newcommand{\bfy}{\ensuremath{\mathbf{y}}}

\newcommand{\calC}{\ensuremath{\mathcal{C}}}
\newcommand{\calD}{\ensuremath{\mathcal{D}}}

\newcommand{\calL}{\ensuremath{\mathcal{L}}}

\newcommand{\calN}{\ensuremath{\mathcal{N}}}

\newcommand{\R}{\mathbb{R}}

\newtheorem{lem}{Lemma}[section]

\newtheorem{thm}[lem]{Theorem}

\newtheorem{assumption}[lem]{Assumption}

\DeclareMathOperator*{\argmin}{arg\,min}

\makeatletter
\newcommand{\vast}{\bBigg@{4}}
\newcommand{\Vast}{\bBigg@{5}}
\makeatother

\newcommand{\ex}[2]{{\ifx&#1& \mathbb{E} \else
\underset{#1}{\mathbb{E}} \fi \left[#2\right]}}
\newcommand{\pr}[2]{{\ifx&#1& \mathbb{P} \else
\underset{#1}{\mathbb{P}} \fi \left[#2\right]}}

\DeclarePairedDelimiterX{\ip}[2]{\langle}{\rangle}{#1, #2}

\newcommand{\privT}{\theta^{\sf priv}\,}

\newcommand{\ltwo}[1]{\left\|#1\right\|_2}
\newcommand{\ltwosq}[1]{\left\|#1\right\|_2^2}

\newcommand{\norm}[1]{\| #1 \|}
\newcommand{\grad}{\nabla}

\DeclarePairedDelimiterX{\infdivx}[2]{(}{)}{%
  #1\;\delimsize\|\;#2%
}

\newcommand*\samethanks[1][\value{footnote}]{\footnotemark[#1]}

\newcommand{\mypar}[1]{\smallskip
	\noindent{\textbf{{#1}:}}}
	
\renewcommand{\epsilon}{\varepsilon}

\renewcommand{\tilde}{\widetilde}

\newcommand{\Ex}[1]{\mathbb{E}\left[#1\right]}

\newcommand{\dist}{\tau}
\newcommand{\npriv}{n_{\sf priv}}
\newcommand{\npub}{n_{\sf pub}}
\newcommand{\hPsi}{\widehat{\Psi}}
\newcommand{\dpriv}{D_{\sf priv}}
\newcommand{\dpub}{D_{\sf pub}}
\newcommand{\ind}{\mathbb{I}_{p}}

\newcommand{\alphaT}{K}
\newcommand{\clip}[2]{{\sf clip}\left(#1,#2\right)}
\newcommand{\barH}{\bar{H}}

\usepackage{fullpage}
\usepackage[numbers, sort, comma, square]{natbib}

\title{Public Data-Assisted Mirror Descent for Private Model Training}
\author{
Ehsan Amid\thanks{Google.  \texttt{\{eamid, mathews, swaroopram, shuangsong, steinke,  omthkkr, athakurta\}@google.com}}
\and
Arun Ganesh\thanks{UC Berkeley. Part of this work done while the author was an intern at Google. Supported in part by NSF CCF-1816861. \texttt{arunganesh@berkeley.edu}}
\and
Rajiv Mathews\samethanks[1]
\and
Swaroop Ramaswamy\samethanks[1]
\and
Shuang Song\samethanks[1]
\and
Thomas Steinke\samethanks[1]
\and
Vinith M. Suriyakumar\thanks{MIT. Part of this work was done while the author was an intern at Google. \texttt{vinithms@mit.edu}}
\and
Om Thakkar\samethanks[1]
\and
Abhradeep Thakurta\samethanks[1]

}

\begin{document}

\maketitle
\begin{abstract}
In this paper, we revisit the problem of using \emph{in-distribution} public data to improve the privacy/utility trade-offs for differentially private (DP) model training. (Here, public data refers to auxiliary data sets that have no privacy concerns.) We design a natural variant of DP mirror descent, where the DP gradients of the private/sensitive data act as the linear term, and the loss generated by the public data as the mirror map.

We show that, for linear regression with feature vectors drawn from a non-isotropic sub-Gaussian distribution, our algorithm, PDA-DPMD (a variant of mirror descent), provides population risk guarantees that are asymptotically better than the best known guarantees under DP (without having access to public data), when the number of public data samples ($\npub$) is sufficiently large. We further show that our algorithm has  natural ``noise stability'' properties that control the variance due to  noise added to ensure DP.

We demonstrate the efficacy of our algorithm by showing privacy/utility trade-offs on four benchmark datasets (StackOverflow, WikiText-2, CIFAR-10, and EMNIST). We show that our algorithm not only significantly improves over traditional DP-SGD, which does not have access to public data, but to our knowledge is the first to improve over DP-SGD on models that have been pre-trained with public data.    
\end{abstract}

\section{Introduction}
\label{sec:intro}

Differentially Private Stochastic Gradient Descent (DP-SGD)~\citep{song2013stochastic,BST14,DP-DL}, and its variants~\citep{kairouz2021practical} have become the de facto standard algorithms for training machine learning models with differential privacy (DP)~\citep{DMNS}. While DP-SGD is known to be optimal in terms of obtaining both optimal excess empirical risk~\citep{BST14}, and excess population risk~\citep{NEURIPS2020_2e2c4bf7} for convex losses, the obtained error guarantees suffer from an explicit polynomial dependence on the model dimension ($p$). This polynomial dependence significantly impacts the privacy/utility trade-off when $p\geq \npriv$, where $\npriv$ is the number of private training samples. Even empirically, when DP-SGD is used to train large deep learning models, there is a significant drop in accuracy compared to the non-private counterpart~\citep{papernot2020tempered}. 

In this paper, we revisit the problem of  using public data (i.e., data without privacy concerns) to improve the privacy/utility trade-offs for DP model training. 
\emph{Specifically, we design central and federated DP variants of mirror descent~\citep{nemirovsky83} that use the loss function generated by the public data as the mirror map and DP gradients on the private data as the linear term.} For linear regression, we show that the excess population risk \emph{asymptotically} improves over the best known bounds under DP (without access to public data samples)~\cite{BST14,bassily2019private} when $\npub$ is sufficiently large (i.e., a small polynomial in $p$), and the public and private feature vectors are drawn from the same non-isotropic sub-Gaussian distribution. Here, $\npub$ is the number of public data samples. Even if $\npub$ is small, our algorithm generalizes DP-SGD, so it never performs worse than DP-SGD.

Furthermore, we show empirically that our DP variant of mirror descent, assisted with public data, can improve the privacy-utility trade-offs by effectively reducing the variance in the noise added to the gradients in DP model training. We show relative improvements up to $5.3\%$ over DP-SGD models, \emph{pre-trained with the same public data}. 
To our knowledge, this is the first work to demonstrate an increased benefit from public data over just pre-training.
Our empirical results are either on simulated linear regression data, or on standard benchmark datasets like StackOverflow, CIFAR-10, EMNIST, and WikiText-2, and only consider $4\%$ of the training data samples ($0.03\%$ for StackOverflow) as public.

\mypar{Learning Geometry with Mirror Maps} Common to most DP model training algorithms, including DP-SGD, DP-FTRL~\citep{kairouz2021practical}, and our algorithm, is a DP estimator of the gradient of the loss $\nabla_\theta \calL(\theta_t;\dpriv)=\sum_{d\in\dpriv}\nabla_\theta\ell(\theta_t;d)$ generated by the private dataset $\dpriv$ at a given model state $\theta_t\in\mathbb{R}^p$. This estimator adds isotropic Gaussian noise $\calN(0,\sigma^2\ind)$ to $\nabla_\theta \calL(\theta_t;\dpriv)$, where $\sigma$ depends on the privacy parameters $(\epsilon,\delta)$ and the maximum allowable value of $\ltwo{\nabla_\theta\ell(\theta_t;d)}$ (a.k.a.~the clipping norm~\citep{DP-DL}).\footnote{For the ease of presentation, at this point we do not consider the noise due to stochastic mini-batching.} It is well known that for most learning tasks, the set of gradients for $\calL(\theta_t;\dpriv)$ is seldom isotropic~\citep{Guygur,agarwal2019efficient}. Hence, it is natural to wonder if the Gaussian noise in the DP estimator can be made to respect the geometry of the gradients. 

Prior works~\citep{ZWB20, asi2021private,KRRT21} have used public data ($\dpub$) to \emph{explicitly} learn this geometry, mostly in the form of preconditioner matrices~\citep{duchi2011adaptive} to be multiplied to the estimated noisy gradients. In this paper, we take an \emph{implicit} approach towards respecting this geometry, by using the loss $\calL(\theta;\dpub)$ generated by the public data as the mirror map in classical mirror descent. As a first order approximation (formalized in Section~\ref{sec:empEval}), one can view it as doing DP-SGD on $\calL(\theta;\dpriv)$ while using $\calL(\theta;\dpub)$ as a regularizer. This approach has the following advantages: (i) The information of the geometry is ``free'', i.e., one does not need to learn the preconditioner explicitly from the public data, (ii) Unlike prior works~\citep{ZWB20,KRRT21}, one does not need to assume that the gradients of $\calL(\theta;\dpriv)$ lie in a low rank subspace, and 
(iii) It is easier to implement since it does not need to maintain an additional data structure for the preconditioner due to the geometry being implicitly defined. Empirically, our algorithm improves over the state of the art~\citep{asi2021private}.

We note that DP mirror descent has been considered before by \cite{talwar2014private,wang17amd}. Their results are not directly comparable to ours because (i) they do not have access to in-distribution public data,  (ii) as shown in \cite{BST14}, without public data, it is impossible to achieve the bounds we achieve, and (iii) in our experiments, we solve unconstrained optimization problems whereas those works choose the mirror map based on the constraint set rather than the dataset. The utility bounds we prove in this paper also apply to a public data-assisted variant of accelerated mirror descent in \cite{wang17amd}.

\mypar{In-distribution vs. Out-of-distribution Public Data} Prior works have considered  settings where the public data comes from the same distribution as the private data  (a.k.a.~\emph{in-distribution})~\citep{BTT19,ZWB20,KRRT21,asi2021private,Wang_Zhou_2020}, and where they can be different (a.k.a.~\emph{out-of-distribution})~\citep{DP-DL,papernot2016semi,papernot2018scalable,li2021large, liu2021leveraging, YZCL21}.

In the in-distribution setting, it is typical that there are fewer public data samples available than private data samples -- i.e., $\npub\ll\npriv$ -- as it is harder to obtain public datasets than ones with privacy constraints attached. In-distribution public data could come from either altruistic \emph{opt-in} users~\citep{merriman14,avent2017blender} or from users who are incentivized to provide such data (e.g., mechanical turks). Out-of-distribution (OOD) public data may be easier to obtain but can have various degrees of freedom; e.g., the domains of private and public data may not be identical, the representation of some classes may vary, the distributions can be mean shifted, etc. It is usually hard to quantify these degrees of freedom to the extent that we can provide precise guarantees. Hence, we leave this aspect for future exploration, and work with the (idealized) assumption that the public data comes from the same distribution as the private data, or, at least, that the differences between these two distributions are not material. It worth emphasizing that although our utility results are for the in-distribution case, our algorithm can be used \emph{as is} in out-of-distribution settings. In a restricted set of experiments, we do compare with one of the SoTA~\cite{asi2021private} for training with OOD public data, and demonstrate improvements in privacy/utility trade-off.

\mypar{Choice of Empirical Benchmark} Mirror descent as a first step optimizes the mirror map function. In our setting, this corresponds to pre-training on the public loss function $\calL(\theta;\dpub)$ before running the DP optimization procedure on $\calL(\theta;\dpriv)$. Since pre-training on public data is intuitive and easy, we always compare to DP-SGD (and its variants) that have been pre-trained to convergence with the public loss. We show that our algorithm \emph{outperforms} even pre-trained DP-SGD. To our knowledge, ours is the first empirical work that compares to this strong (but fair) benchmark.

\mypar{Other Uses of Public Data in DP Learning}
The use of in-distribution public data has been extensively explored both theoretically and empirically. 
On the theoretical side, it has been shown \citep{alon2019limits,BassilyCMNUW20} that a combination of private and public data samples can yield asymptotically better worst-case PAC learning guarantees than either on their own.
Another line of work \citep{papernot2016semi,papernot2018scalable,bassily2018model,dwork2018privacy,NandiB20} considers public data that is unlabelled, but otherwise comes from the same distribution as the private data; the primary goal is to use the private data to generate labels for the public data, which can then be used arbitrarily. Additionally,~\cite{feldman2018privacy} showed that for convex ERMs, using $\approx p$ in-distribution public data samples, one can obtain dimension independent population risk guarantees. However, the main tool used to prove DP (i.e., privacy amplification by iteration) heavily relies on convexity. As a result, their algorithm is inapplicable to the deep learning problems we consider in this paper.

So far only two papers have considered out-of-distribution data from a theory standpoint. \cite{bassily2020learning} assume that whether a data record is public or private depends on its label; e.g., the public data may contain many negative examples, but few positive examples. They show that halfspaces can be learned in this model. \cite{liu2021leveraging} consider synthetic data generation and provide guarantees that depend on the R\'enyi divergences between the public and private distributions.
\cite{DP-DL,tramer2020differentially} provided techniques to effectively use out-of-distribution public data for pre-training for DP-SGD. However, they did not consider techniques to improve a pre-trained model using private and public data, which is the focus of our work.
A recent work \cite{YZCL21} uses public data to dynamically adjust the privacy budget and clipping norm. Our technique crucially uses the public data to learn the geometry of the gradients;~\cite{YZCL21} is complementary to ours and can be utilized for potential additional gains from using the public data after pre-training. 

\subsection{Problem Formulation}
\label{sec:form}

Consider the classic DP stochastic convex optimization (DP-SCO)~\citep{chaudhuri2011differentially,BST14,bassily2019private,NEURIPS2020_2e2c4bf7} setting. Let $\dist$ be a distribution over a fixed domain $\calD$. Given a dataset $D\in\calD^*$ drawn i.i.d. from $\dist$, and a loss function $\ell_{\sf priv}:\mathbb{R}^p\times\calD\to\mathbb{R}$, the objective is to approximately solve $\argmin\limits_{\theta\in\calC}\mathbb{E}_{d\sim\dist}\left[\ell_{\sf priv}(\theta;d)\right]$, while preserving DP. Here, $\calC\subseteq\mathbb{R}^p$ is the constraint set.
Usually one solves the SCO problem via empirical risk minimization (ERM), i.e., $\privT\in\argmin\limits_{\theta\in\calC}\calL(\theta;D)$, where $\calL(\theta;D)=\frac{1}{|D|}\sum\limits_{d\in D}\ell_{\sf priv}(\theta;d)$, and then uses $\privT$ as a proxy.
Up to a dependence on dimensionality $p$, in the DP setting, a direct translation from ERM to the SCO setting provides optimal rates~\citep{BST14,bassily2019private,NEURIPS2020_2e2c4bf7}. 

We consider the DP-SCO setting with \emph{heterogeneous data}, where there are two datasets $D_{\sf priv}$ (with $n_{\sf priv}$ samples) and $D_{\sf pub}$ (with $\npub$ samples) drawn i.i.d. from the \emph{same distribution}. The private dataset $D_{\sf priv}$ requires privacy protection, whereas the public dataset $D_{\sf pub}$ does not. 
Since obtaining such data can be expensive, for our empirical evaluation,
we assume $n_{\sf pub} \ll n_{\sf priv}$ (e.g., $n_{\sf pub} \le \frac{1}{20} n_{\sf priv}$). 

Our algorithm allows the usage of a separate public loss function $\ell_{\sf pub}$. As a simple demonstration, we give a theoretical analysis where $\ell_{\sf priv}$ and $\ell_{\sf pub}$  both correspond to the linear regression loss $\frac{1}{2}(y-\ip{\bfx}{\theta})^2$. In practice too, one will likely choose $\ell_{\sf priv} = \ell_{\sf pub}$, but we may clip the gradients of $\ell_{\sf priv}$ for privacy. In general, $\ell_{\sf pub}$ can be arbitrary. 

We refer the reader to Appendix~\ref{sec:notation} for a reference for the notation used throughout the paper.

\subsection{Our Contributions}
\label{sec:contrib}

\mypar{Algorithm} Our algorithm, Public Data Assisted Differentially Private Mirror Descent (PDA-DPMD), is similar to DP-SGD but utilizes the public data in two ways. First, we can pre-train on the public data to obtain a better starting point for training. Second, we use mirror descent, with the public loss function as the mirror map, to reshape the noisy gradients used in DP-SGD. In doing this, PDA-DPMD effectively takes smaller gradient steps and adds less noise in directions where the public loss grows quickly.

\mypar{Tighter Excess Population Risk for Linear Regression} We consider the standard setting of linear regression where the loss function is $\ell(\theta;d)=\frac{1}{2}(y-\ip{\bfx}{\theta})^2$, with data sample $d=(\bfx,y)$. Let $\tau$ be the data generating distribution and $\theta^*=\argmin\limits_{\theta\in\calC}\mathbb{E}_{d\sim\tau}[\ell(\theta;d)]$ be the population minimizer. We assume a uniform bound on the feature vectors of the form $\ltwo{\bfx}\leq 1$, and on the response $|y-\ip{\bfx}{\theta^*}|\leq 1$. Suppose the feature vectors are drawn i.i.d. from a distribution with covariance matrix $\barH$.
In this setting, DP-SGD obtains an error of roughly $\frac{p}{\lambda_{\min}(\barH)\epsilon^2 \npriv^2} + \frac{1}{\lambda_{\min}(\barH) \npriv}$. 
If we use PDA-DPMD instead, we can show that given a sufficient number of public samples, the first term depends on the \textit{average} rather than the \textit{minimum} eigenvalue. 
For example, if $\barH$ has one eigenvalue being $1/p^{1.5}$ and the remaining eigenvalues being $1/p$, then with $n_{\sf pub} = \tilde{\Omega}(p^{2.5})$ public samples, PDA-DPMD obtains an error of $\frac{p^2}{\epsilon^2 \npriv^2} + \frac{p^{1.5}}{ \npriv}$, whereas DP-SGD gets $\frac{p^{2.5}}{\epsilon^2 \npriv^2} + \frac{p^{1.5}}{ \npriv}$. Since PDA-DPMD generalizes DP-SGD, unsurprisingly, it still recovers the error bound of DP-SGD in the isotropic case.
We provide the formal statement in Theorem~\ref{thm:linreg}.

\mypar{Local Noise-stability}
We show that in addition to achieving better excess population loss bounds, PDA-DPMD has the following ``local noise-stability'' property: If in a bounded region around the current model $\theta_t$, the public loss is $\lambda_\bfv$-strongly convex in a direction $\bfv$, then using noisy gradients instead of the exact gradients shifts $\theta_{t+1}$ in the direction $\bfv$ by an amount proportional to $1/\lambda_\bfv$ (see Theorem~\ref{thm:locstab-linreg} for the formal statement). That is, PDA-DPMD effectively rescales the amount of noise added in any direction to be inversely proportional to the curvature in that direction. Note that this is in spite of the fact that for privacy, the noise we add to the gradients is usually isotropic. Furthermore, PDA-DPMD can perform this rescaling using only a gradient oracle for the public loss function. In other words, a practitioner implementing the algorithm simply needs to choose an appropriate loss function, and PDA-DPMD will ``automatically'' rescale the effects of noise to match the loss function's curvature.

\mypar{Empirical Evaluation} On both synthetic and real-world benchmark datasets, we show that PDA-DPMD outperforms DP-SGD, even when they are pre-trained on the public dataset. We provide two sets of experiments. First, a linear regression on a synthetic dataset which closely matches the utility assumptions in the theoretical analysis. Second, we provide results on deep learning benchmark datasets (StackOverflow, WikiText-2, CIFAR-10, and EMNIST). 

In Section~\ref{sec:dpmd}, we consider using DP-SGD and PDA-DPMD to solve a least squares linear regression problem on a synthetic data set generated via the process $y_i \sim N(\ip{\bfx_i}{\theta^*}, \sigma^2)$, where $\theta^*\in\mathbb{R}^p$ is the true model. The feature vectors $\bfx_i$'s are drawn i.i.d. from some fixed distribution.  We fix the number of private data samples, and set the number of public samples to be a fixed constant times the dimension ($p$).  We observe that as expected, public data allows us to substantially improve the error in two ways: (i) {\bf Pre-training}: DP-SGD initialized from a model pre-trained on public data has nearly-constant mean-squared error, whereas DP-SGD from a random initialization has mean-squared error scaling with the dimension, and (ii) {\bf Adapting to geometry using public loss}: While DP-SGD initialized from a pre-trained model already achieves near-constant loss, we also observe that PDA-DPMD outperforms DP-SGD due to its error's dependence on the Gaussian width $G_Q$ rather than the dimension. 
We note that the observed improvement is ``automatic'' once we choose the mean-squared error to be the loss function.

For the deep learning experiments, since running PDA-DPMD can be computationally expensive, we derive a first-order approximation that can be viewed as DP-SGD on a convex combination of the private and  public losses.
This makes the running time of our algorithm comparable to DP-SGD when run on the dataset $\dpriv\cup\dpub$.

For user-level DP, we conduct experiments in Section~\ref{sec:empEvalFederated} for next word prediction (NWP) on the StackOverflow dataset.
We consider $\sim$0.03\% of the original training users as public and pre-train using them. We find that  PDA-DPMD tailored towards a user-level setting (e.g., based on DP Federated Averaging (DP-FedAvg)~\cite{dplangmodels}) outperforms DP-FedAvg: PDA-DPMD obtains a  0.57\% absolute (2.7\% relative) increase in accuracy, and a 5.3\% drop in  perplexity.
Note that PDA-DPMD can directly apply to Federated Learning \cite{FL1} settings since an orchestrating server does not need to provide  public data to any participating private client.

For sample-level DP, we conduct experiments in Section~\ref{sec:empEvalCentral} on two real world tasks: NWP on WikiText-2, and image classification on CIFAR-10 and EMNIST. We consider 4\% of the original training data as public and pre-train on it. On all datasets, we can observe that PDA-DPMD outperforms DP-SGD in terms of test loss. On CIFAR-10, the improvement is more than 5\%; on EMNIST, 7\%; on WikiText-2, log perplexity is improved by more than 0.3\%, which is a notable improvement for perplexity.

\subsection{Organization}

In Section~\ref{sec:back} we provide background details on differential privacy and mirror descent. In Section~\ref{sec:dpmd} we discuss the main algorithmic contribution, and the corresponding privacy guarantee. In Section~\ref{sec:linreg} we provide the privacy/utility trade-offs for linear regression. In Section~\ref{sec:empEval} we provide a detailed empirical evaluation. Additionally, in Appendix~\ref{sec:notation} we provide a table of notions.
\section{Background}\label{sec:back}
\mypar{Differential Privacy}
\label{sec:priv}
Differential Privacy (DP)~\cite{DMNS} is a formal method for quantifying the privacy leakage from the output of a data analysis procedure.
A randomized algorithm $M : \mathcal{D}^* \to \mathcal{Y}$ is $(\varepsilon,\delta)$-DP if, for all neighbouring dataset pairs $D,D' \in \mathcal{D}^*$ and all measurable sets of outputs $S \subseteq \mathcal{Y}$, we have \[\pr{}{M(D) \in S} \le e^\varepsilon \cdot \pr{}{M(D') \in S} + \delta.\]
We define two datasets to be neighbouring if they differ only by the \emph{addition or removal of one person's record}. 
We ensure differential privacy by adding Gaussian noise to functions of bounded sensitivity. In particular, if $\ell$ is $L$-Lipschitz in its first parameter, then $\| \nabla_\theta \ell(\theta;d)\|_2 \le L$ for all $\theta$ and $d\in\calD$. Thus adding noise drawn from $\mathcal{N}(0,\sigma^2\cdot\ind)$ to the sum $\sum\limits_{i}  \nabla_\theta \ell(\theta, d_i)$ over people's records satisfies DP, where $\sigma$ scales with $L$ and the desired privacy parameters.
The composition and postprocessing properties of differential privacy ensure that, as long as each step in our iterative algorithm satisfies differential privacy, then so does the overall system. We refer the reader to \citep{dwork2014algorithmic} for further details of the standard privacy analysis of algorithms like ours.

\mypar{Mirror Maps}
\label{sec:mirror}
A mirror map is a differentiable function $\Psi : \mathbb{R}^p \to \mathbb{R}$ that is strictly convex. Since $\Psi$ is strictly convex and differentiable, $\grad \Psi : \mathbb{R}^p \to \mathbb{R}^p$ provides a bijection from $\mathbb{R}^p$ to itself. One can view $\theta$ as lying in a primal space and $\grad \Psi(\theta)$ as lying in a dual space. In turn, we could now consider optimizing over the value $\grad \Psi(\theta)$ in the dual space instead of $\theta$ primal space. Mirror descent does exactly that, performing gradient descent in the dual space by computing the gradient $\bfg_t = \grad \ell(\theta_t)$ (where $\theta_t$ lies in the primal space), taking a step in the opposite direction in the dual space, and then using the inverse of the mirror map to determine $\theta_{t+1}$. Mirror descent is essentially motivated as minimizing a (linearized) loss plus a Bregman divergence (induced by $\Psi$) as the regularizer~\citep{nemirovsky83}.
More formally, similar to proximal gradient descent, mirror descent is equivalent to taking the gradient $\bfg_t$ and performing the update $\theta_{t+1} = \argmin_{\theta \in \calC}[\eta \langle \bfg_t, \theta \rangle + B_{\Psi}(\theta, \theta_t)]$ where $B_{\Psi}(\theta_1, \theta_2)=\Psi(\theta_1) - \Psi(\theta_2) - \langle \grad \Psi(\theta_2), \theta_1 - \theta_2 \rangle$ is the Bregman divergence generated by $\Psi$. Note that, if $\Psi(\theta)=\|\theta\|_2^2$, then the Bregman divergence is simply $B_{\Psi}(\theta_1, \theta_2)=\|\theta_1 - \theta_2\|_2^2$ and mirror descent is equivalent to the usual gradient descent.

\mypar{Gaussian Width}
Given a bounded set $Q \subset \mathbb{R}^d$, the Gaussian width of $Q$, $G_Q$, is a measure of how isotropic the set is. $G_Q$ is defined as $\mathbb{E}_{g \sim N(0, \mathbb{I}_p)} \max_{x \in Q} \langle g, x \rangle$.  Although the Gaussian width is well-defined for any bounded set, to gain intuition it suffices to consider defining the Gaussian width of convex sets containing the origin such that $\max_{x \in Q} \ltwo{x} = 1$; rescaling any such set by a constant changes the Gaussian width by the same constant. If $Q$ is just the unit $\ell_2$-ball, the ``most isotropic'' set satisfying this condition, then we have $G_Q = \sqrt{p}$; in particular, since every set $Q$ satisfying $\max_{x \in Q} \ltwo{x} = 1$ is contained in the $\ell_2$-ball, this is the maximum Gaussian width of any such set. On the other hand, if $Q$ is just the line from the origin to a single unit vector, we have $G_Q = \Theta(1)$. More generally, for any ellipsoid centered at the origin whose axes have radii $0 \leq r_i \leq 1, 1 \leq i \leq p$, we have that the Gaussian width of this ellipsoid is $\Theta(\sqrt{\sum_{i=1}^p r_i^2})$. As other examples, the Gaussian width of the $\ell_1$-ball of radius 1 is roughly $\log p$, and the Gaussian width of the $\ell_\infty$ ball of radius $1/\sqrt{p}$ is roughly $\sqrt{p}$. 
\section{Algorithm Description}
\label{sec:dpmd}

In this section, we present our main algorithm Public Data-Assisted Differentially Private Mirror Descent (PDA-DPMD). Given in Algorithm~\ref{alg:dpmd}, it is a variant of mirror descent using noisy gradients, but we also pre-train on public data and use the public loss as our mirror map $\Psi$. 

\begin{algorithm}
\caption{Public Data-Assisted Differentially Private Mirror Descent (PDA-DPMD)}
\textbf{Input:} Public/private datasets $D_{\sf pub}, D_{\sf priv}$ of sizes $n_{\sf pub}, n_{\sf priv}$, private/public loss functions $\ell_{\sf priv}$, $\ell_{\sf pub}$, privacy parameters $(\epsilon, \delta)$, number of iterations $T$, learning rate $\eta:\{0, 1, \ldots, T-1\} \rightarrow \mathbb{R}^+$, constraint set: $\calC$, clipping norm $L$: an upper bound on $\max\limits_{\theta\in\calC}\ltwo{\grad\ell_{\sf priv}(\theta)}$
\begin{algorithmic}[1]
\State{$\Psi(\theta) := \frac{1}{n_{\sf pub}}\sum\limits_{d \in D_{\sf pub}} \ell_{\sf pub}(\theta; d)$} 
\State{$\theta_0 \leftarrow \argmin_{\theta \in \calC} \Psi(\theta)$, $\sigma^2 \leftarrow \frac{8 L^2 T \log(1/\delta)}{(\epsilon n_{\sf priv})^2}$}
\For{$t = 0, \ldots, T-1$}
        \State{$\bfg_t \leftarrow  \frac{1}{n_{\sf priv}} \sum\limits_{d \in D_{\sf priv}} \clip{\nabla \ell_{\sf priv}(\theta; d)}{L}$, where $\clip{\bfv}{L}=\bfv\cdot\min\left\{1,\frac{L}{\ltwo{\bfv}}\right\}$}
        \State{$\theta_{t+1} \leftarrow \argmin_{\theta \in \mathcal{C}}\left[\eta_t \langle \bfg_t + \bfb_t, \theta \rangle + B_\Psi(\theta, \theta_t)\right] $, where $\bfb_t \sim \calN(0, \sigma^2 \cdot \ind)$}\label{line:mdupdate}
\EndFor
\State{\Return{$\theta_{\sf priv} := \frac{1}{T} \sum\limits_{t=1}^T \theta_t$}}
\end{algorithmic}
\label{alg:dpmd}
\end{algorithm}

Note that Line~\ref{line:mdupdate} of PDA-DPMD is equivalent to the following: Choose $\theta_{t+1/2}$ to be the point such that $\grad \Psi(\theta_{t+1/2}) = \grad \Psi(\theta_{t}) - \eta (\bfg_t + \bfb_t)$, and then use the Bregman projection $\theta_{t+1} = \argmin_{\theta \in \mathcal{C}} B_\Psi(\theta, \theta_{t+1/2})$. Intuitively, PDA-DPMD is similar to DP-SGD, with the main difference being we apply the gradient steps to $\grad \Psi(\theta)$ rather than to $\theta$ itself. Note that PDA-DPMD reshapes the gradient and noise \textit{automatically}  given $\ell_{\sf pub}$ and $D_{\sf pub}$. In contrast, e.g., private AdaGrad implementations \citep{KRRT20, asi2021private} assume knowledge of the geometry of the loss function has already been learned prior to running their algorithms. Also, for an appropriate choice of $\Psi$, one can recover an algorithm that projects the private gradients to a low-dimensional subspace as in the algorithms of \citet{ZWB20} and \cite{KRRT20}. From e.g. \cite{DP-DL} we have the privacy guarantee for Algorithm~\ref{alg:dpmd}:

\begin{thm}\label{thm:dpmd-privacy}
Algorithm~\ref{alg:dpmd} (PDA-DPMD) is $(\epsilon, \delta)$-DP with respect to the private dataset $D_{\sf priv}$.
\label{thm:priv}
\end{thm}

\section{Case Study: Linear Regression}
\label{sec:linreg}

In this section, we apply Algorithm~\ref{alg:dpmd} (PDA-DPMD) to linear regression -- an important example that is still amenable to theoretical analysis. We prove utility guarantees, with supporting simulations.

\mypar{Problem setup} Given a data sample $d_i = (\bfx_i, y_i)$, the loss of a model $\theta$ is defined as $\ell(\theta; d_i) := \frac{1}{2}(y_i - \ip{\theta}{\bfx_i})^2$. Consider two datasets drawn i.i.d.~from a fixed distribution $\tau$: i) The public dataset $\dpub$ with $\npub$ data samples, and ii) The private dataset $\dpriv$ with $\npriv$ data samples. In this section, we will denote both the public and private loss functions ($\ell_{\sf pub}$ and $\ell_{\sf priv}$ respectively in Algorithm~\ref{alg:dpmd}) by $\ell$.

\begin{assumption}
We assume that we are given an initial constraint set\footnote{The assumption that $\calC_0$ is centered at the origin is without loss of generality.} $\calC_0 = \{\theta: \ltwo{\theta} \leq r\}$ with $r = O(1)$ that contains the population minimizer, i.e., $\theta^*=\argmin\limits_{\theta\in\mathbb{R}^p}\mathbb{E}_{d\sim\tau}{\ell(\theta;d)} \in \calC_0$.  We further assume that for each feature vector $\ltwo{\bfx}\leq 1$, and for each response $|y - \langle \theta^*, \bfx \rangle|\leq 1$. Let $\barH$ be the Hessian of the loss function $\mathbb{E}_{d\sim\tau}\left[\ell(\theta;d)\right]$. In terms of data set sizes, we assume that $\npriv \geq \npub$ and $\npub=\Omega\left(\frac{\log(p/\delta)}{\lambda_{\min}(\bar{H})}\right)$.
\label{ass:12}
\end{assumption}

\mypar{Excess population risk guarantees} In Theorem~\ref{thm:linreg} we first provide the excess population risk guarantee for Algorithm~\ref{alg:dpmd} (PDA-DPMD) under Assumption~\ref{ass:12}. Furtheremore, in some regimes we demonstrate asymptotic improvement over standard privacy/utility trade-offs for algorithms without access to public data samples.

\begin{thm}\label{thm:linreg}
Consider Assumption~\ref{ass:12}. We run Algorithm~\ref{alg:dpmd} (PDA-DPMD) using  $L=O(1)$, constraint set $\calC =\left\{\theta \in \calC_0: \ltwo{\grad \Psi(\theta)} = O\left(\sqrt\frac{{\log(1/\delta)}}{{\npub}}\right) \right\}$, and an appropriate choice of $\eta_t$ and $T$. Let
{\small
\[ \chi = \max\left\{\frac{1}{\lambda_{\min}(\bar{H})}, \lambda_{\max}(\bar{H}) \npub\right\} \cdot\sum_i \min\left\{1, \frac{\log(1/\delta)}{\lambda_i(\bar{H})^2 n_{\sf pub}}\right\}.\]
}
Then, Algorithm~\ref{alg:dpmd} is $(\epsilon, \delta)$-DP and we have the following guarantee on $\calL(\theta) := \mathbb{E}_{d \sim \tau}\left[\ell(\theta; d)\right]$:
{\small
\[ \mathbb{E}\left[\mathcal{L}(\theta_{\sf priv}) - \mathcal{L}(\theta^*)\right]\leq \Tilde{O}\left(\frac{\chi \log(1/\delta) }{\epsilon^2 n_{\sf priv}^2} + \frac{1}{\lambda_{\min}(\bar{H}) n_{\sf priv}} \right).\]
}
The expectation is over $D_{\sf pub}, D_{\sf priv}$, and the algorithm. $\widetilde O(\cdot)$ hides polylog factors in $n_{\sf priv}, n_{\sf pub}$ and $\lambda_{\min}(\bar{H})$.
\end{thm}

We note that the idea of using public data to shrink the constraint set $\calC$ is similar to the idea used by \cite{biswas2020coinpress} for mean estimation, though their result iteratively uses each private mean estimate to shrink the constraint set before re-estimating the mean, as opposed to our one-shot approach to shrinking the constraint set using public data.

To interpret $\chi$ in Theorem~\ref{thm:linreg}, note that a natural setting of parameters to consider would be where the feature vectors (i.e., the $\bfx$'s) are coming from a mean-zero truncated Gaussian distribution with covariance $\frac{1}{p}\cdot\mathbb{I}$. In this case, all $\lambda_i$ are $1/p$. If $\npub = \widetilde \Omega(p)$, then $\chi$ evaluates to $p^2$, and so we get a bound of $\widetilde O\left(\frac{p^{2}}{\epsilon^2 \npriv^2} + \frac{p}{\npriv}\right)$, matching the excess population risk of DP-SGD. Note that one can still recover DP-SGD's loss bound with Algorithm~\ref{alg:dpmd} even if $\npub=O(1)$ by instead setting $\Psi$ to be $\frac{1}{2}\ltwo{\theta}^2$ and $\calC = \calC_0$.

One can also consider a non-isotropic setting, where $\lambda_{\min}(\bar{H})$ is $1/p^{1.5}$ instead of $1/p$, but all other eigenvalues remain roughly $1/p$. In this setting, DP-SGD would give an error bound of $\widetilde O\left(\frac{p^{2.5}}{\epsilon^2 \npriv^2} + \frac{p^{1.5}}{\npriv}\right)$. If $\npub = \widetilde \Omega(p^{3/2})$, we again match the DP-SGD bound. If instead we have $\npub = \widetilde \Omega(p^c)$ for $2 \leq c \leq 2.5$, then $\chi$ in our loss bound becomes $p^{4.5-c}$, and our loss bound becomes  $\widetilde O\left(\frac{p^{4.5-c}}{\epsilon^2 \npriv^2} + \frac{p^{1.5}}{\npriv}\right)$. Once $c = 2.5$, the first term becomes $\frac{p^2}{\epsilon^2 \npriv^2}$, matching the corresponding term for the isotropic setting. This shows that PDA-DPMD asymptotically improves over DP-SGD under a non-isotropic geometry when given sufficiently many public data samples, with the improvement increasing as the number of public samples increases.

We now turn to proving Theorem~\ref{thm:linreg}. We first show that the set $\calC$ contains $\theta^*$ with high probability. To do this, we need a bound on the gradients of $\ell$ at $\theta^*$.

\begin{lem}\label{lemma:optlipschitz}
Under Assumption~\ref{ass:12}, for all $d \in supp(\tau)$ we have $\ltwo{\grad \ell(\theta^*; d)} \leq 1$.
\end{lem}
\begin{proof}
The loss function for the pair $d = (\bfx, y)$ is $\ltwosq{\bfx}$-smooth, and minimized (i.e. has gradient $\boldzero$) at the point $\theta^* + \frac{y - \langle \theta^*, \bfx \rangle}{\ltwosq{\bfx}} \bfx$. In turn, by smoothness and Assumption~\ref{ass:12} we have:

\[\ltwo{\grad \ell(\theta^*; d)}= \ltwo{\grad \ell(\theta^*; d) - \grad \ell(\theta^* + \frac{y - \langle \theta^*, \bfx \rangle }{\ltwosq{\bfx}} \bfx; d)}\leq \ltwosq{\bfx} \cdot \left| \frac{y - \langle \theta^*, \bfx \rangle}{\ltwosq{\bfx}}\right| \cdot \ltwo{\bfx} \leq 1.\]
\end{proof}

We can now show that the gradient of the public loss evaluated at $\theta^*$ is bounded with high probability, implying it is in $\calC$.

\begin{lem}\label{lemma:pointvariance}
With probability at least $1 - \delta$, for $\Psi$ as defined in Algorithm~\ref{alg:dpmd}, we have $\ltwo{\grad \Psi(\theta^*)} \leq O(\frac{\sqrt{\log(1/\delta)}}{\sqrt{n_{\sf pub}}}).$
\end{lem}
\begin{proof}
Since $\theta^*$ is the population minimizer of $\ell$ in $\mathbb{R}^p$, and $\mathbb{E}_{d \sim \tau}[\ell(\theta; d)]$ is strongly convex, we have $\mathbb{E}_{d \sim \tau}[\grad \ell(\theta^*; d)] = \boldzero$. The lemma now follows from a vector Azuma inequality (see e.g. \cite{hayes03vectorazuma}) applied to the vector sum $\grad \Psi(\theta^*)$, and Lemma~\ref{lemma:optlipschitz}, which gives that $\ltwo{\grad \ell(\theta^*; d)} \leq 1$ for all $d \in supp(\tau)$.
\end{proof}

We can also use the bound on the gradients $\grad \ell(\theta^*; d)$ to show every loss function is Lipschitz within the constraint set.

\begin{lem}\label{lemma:lipschitzbound}
For all $d$, $\ell(\theta; d)$ is $L$-Lipschitz within $\calC_0$ for $L = O(1)$.
\end{lem}
\begin{proof}
Each $\ell(\theta; d)$ is $1$-smooth, we have $\theta^* \in \calC$. In turn, by smoothness and Lemma~\ref{lemma:optlipschitz}, each $\ell(\theta; d)$ is $L$-Lipschitz for $L = 1 + 2\ltwo{\calC_0}$, which is $O(1)$ under Assumption~\ref{ass:12}, giving the lemma.
\end{proof}

We now show that the sample Hessian approximates the population Hessian for both $D_{\sf priv}$ and $D_{\sf pub}$, i.e. the geometry of $\Psi$ matches the population loss' geometry and the private sample loss' geometry.

\begin{lem}\label{lemma:hessianapprox}
Let $\hat{H}_{\sf pub}$ be the Hessian of the public loss function $\Psi$, and $\hat{H}_{\sf priv}$ be the Hessian of the private loss function $\frac{1}{n_{\sf priv}} \sum_{d \in D_{\sf priv}} \ell(\theta; d)$. Then under Assumption~\ref{ass:12} with probability $1 - \delta$, we have 

\[\frac{1}{2} \bar{H} \preccurlyeq \bar{H} - \frac{\lambda_{\min}(\bar{H})}{2} \mathbb{I} \preccurlyeq \hat{H}_{\sf pub} \preccurlyeq \bar{H} + \frac{\lambda_{\min}(\bar{H})}{2} \mathbb{I} \preccurlyeq 2 \bar{H},\]
\[\frac{1}{2} \bar{H} \preccurlyeq \bar{H} - \frac{\lambda_{\min}(\bar{H})}{2} \mathbb{I} \preccurlyeq \hat{H}_{\sf priv} \preccurlyeq \bar{H} + \frac{\lambda_{\min}(\bar{H})}{2} \mathbb{I} \preccurlyeq 2 \bar{H}.\]
\end{lem}
\begin{proof}
The outer inequalities $\frac{1}{2} \bar{H} \preccurlyeq \bar{H} - \frac{\lambda_{\min}(\bar{H})}{2} \mathbb{I}$ and $\bar{H} + \frac{\lambda_{\min}(\bar{H})}{2} \mathbb{I} \preccurlyeq 2 \bar{H}$ follow from the $\lambda_{\min}(\bar{H})$-strong convexity of the population loss, i.e. $\lambda_{\min}(\bar{H}) \mathbb{I} \preccurlyeq \bar{H}$. So it suffices to prove the inner inequalities.

Let $H_d$ be the Hessian of $\ell(\theta; d)$. By 1-smoothness of $H$ and $\lambda_{\min}(\bar{H})$-strong convexity of $\bar{H}$, we have:

\[ \boldzero \preccurlyeq H_d \preccurlyeq \mathbb{I} \qquad \forall d,\]
\[ \boldzero \preccurlyeq \bar{H} \preccurlyeq  \mathbb{I}.\]

And so:

\[ -\mathbb{I} \preccurlyeq H_d - \bar{H} \preccurlyeq \mathbb{I} \qquad \forall d.\]

The inner inequalities now follow from a matrix Bernstein inequality, and the sample complexity lower bounds given in Assumption~\ref{ass:12}.
\end{proof}

We can now prove our main result.

\begin{proof}[Proof of Theorem~\ref{thm:linreg}]
Algorithm~\ref{alg:dpmd} is $(\epsilon, \delta)$-DP by Theorem~\ref{thm:dpmd-privacy}.

For the utility guarantee, with probability at most $3\delta$, one of the high probability events described in Lemmas~\ref{lemma:pointvariance} and~\ref{lemma:lipschitzbound} fails to hold. In this case, by e.g., Lemma~\ref{lemma:lipschitzbound} we can use a naive bound of $O(\ltwo{\calC_0})$ on the loss. Since $\delta$ is negligible, the contribution of this case to the expected excess loss is negligible, so we ignore it here. We now wish to follow the analysis of Theorem A.1 in \cite{talwar2014private}. To do so, we need to calculate various parameters in that theorem statement:

\begin{itemize}
    \item By $\lambda_{\min}(\bar{H})$-strong convexity of $\Psi$, $\ltwo{\calC} =  O(\min\{1,  \frac{\sqrt{\log(1/\delta)}}{\lambda_{\min}(\bar{H}) \sqrt{n_{\sf pub}}}\})$.
    \item We can assume without loss of generality that $\ltwo{\theta} \leq r/2$. This is because if we replace $r$ with $2r$ in the definition of $\calC_0$, the parameters of the problem do not change asymptotically, but this condition is now enforced. Under this assumption, any line passing through $\theta^*$ has an intersection with $\calC_0$ of length $\Omega(1)$. Now, by strong convexity and the definition of $\calC$, this implies $\calC$ is contained within an ellipsoid $\tilde{Q}$ whose axes are the eigenvectors of $\hat{H}_{\sf pub}$, and whose axis lengths are $\Theta(\min\{1, \frac{\sqrt{\log(1/\delta)}}{\lambda_i \sqrt{n_{\sf pub}}}\})$. Furthermore, $\calC$ contains $\tilde{Q}$ rescaled in all dimensions by a constant. This means the symmetric convex hull $Q$ of $\calC$ is also contained in $\tilde{Q}$, and contains $\tilde{Q}$ rescaled by a constant. So the strong convexity of $\Psi$ with respect to the $Q$-norm is within a constant factor of the strong convexity of $\Psi$ with respect to the $\tilde{Q}$-norm.
    
    Now, let $\norm{\cdot}_{\tilde{Q}}$ be the Minkowski $\tilde{Q}$-norm $\norm{\bfx}_{\tilde{Q}} = \min\{a \in \mathbb{R}_{\geq 0}: \bfx \in a\tilde{Q}\}$. 
    In the direction of the $i$th eigenvector, $\Psi$ is $\frac{1}{\lambda_{i}(\hat{H}_{\sf pub})}$-strongly convex with respect to the norm $\norm{\cdot}_{Q'}$ for the set ${Q'} = \{\theta \in \mathbb{R}^p : \ltwo{\grad \Psi(\theta)} \leq 1\}$, so it is $\Theta(\frac{\min\{\lambda_i(\hat{H}_{\sf pub})^2, \log(1/\delta)/n_{\sf pub}\}}{\lambda_i(\hat{H}_{\sf pub})})$-strongly convex with respect to the $\tilde{Q}$-norm, and thus the $Q$-norm, in this direction. So $\Delta$, the strong convexity parameter of $\Psi$ with respect to the $Q$-norm is:
    \[\Delta = \Theta\left(\min_i\left\{\frac{\min\{\lambda_i(\hat{H}_{\sf pub})^2, \log(1/\delta)/n_{\sf pub}\}}{\lambda_i(\hat{H}_{\sf pub})}\right\}\right) = \Theta\left(\min\left\{ \lambda_{\sf \min}(\hat{H}_{\sf pub}), \frac{\log(1/\delta)}{\lambda_{\max}(\hat{H}_{\sf pub}) n_{\sf pub}} \right\}\right)\]
    
    By Lemma~\ref{lemma:hessianapprox}, conditioned on the event in that lemma $\Delta$ is \[\Theta\left(\min\left\{ \lambda_{\sf \min}(\bar{H}), \frac{\log(1/\delta)}{\lambda_{\max}(\bar{H}) n_{\sf pub}} \right\}\right).\]
    \item By a similar argument to the previous item, we get that the squared Gaussian width $G_\calC^2$ is at most $G_{\tilde{Q}}^2$, which is 
    \[O\left(\sum_i \min\left\{1, \frac{\log(1/\delta)}{\lambda_i(\bar{H})^2 n_{\sf pub}}\right\}\right)\]
    \item By Lemma~\ref{lemma:hessianapprox}, conditioned on the event in that lemma, the Hessians of the public sample loss, private sample loss, and population loss are constant-approximations of each other.. From the definition of strong convexity with respect to a function (see Section 2.2 of \cite{talwar2014private}), any quadratic function is 1-strongly convex with respect to itself, and in turn 
    $\Theta(1)$-strongly convex with respect to another quadratic function whose Hessian is within a constant factor of its own, since this implies the Bregman divergences induced by the two functions are also within a constant factor. So the sample private loss $\frac{1}{n_{\sf priv}} \sum_{d \in D_{\sf priv}} \ell(\theta; d)$ is $\Theta(1)$-strongly convex with respect to $\Psi$. 
\end{itemize}

We will view Algorithm 1 as equivalently using $\Psi' = \frac{1}{\Delta} \Psi$ in place of $\Psi$, and $\eta_t' = \Delta \eta_t$ in place of $\eta_t$. $\Psi'$ is 1-strongly convex with respect to the $Q$-norm, and the sample private loss is now $\Theta(\Delta)$-strongly convex with respect to $\Psi'$.
Now, following the proof of Theorem A.1 in~\cite{talwar2014private}, setting $\eta_t' = 1/\Delta t$ and $T = \frac{\ltwosq{\calC}(\epsilon n_{\sf priv})^2}{\ltwosq{\calC} + G_\calC^2}$, conditioned on the high probability events we get an excess empirical loss bound of:

\[\Tilde{O}\left(\frac{\log(1/\delta) \max\{\frac{1}{\lambda_{\min}(\bar{H})}, \lambda_{\max}(\bar{H}) \npub\} \cdot \sum_i \min\left\{1, \frac{\log(1/\delta)}{\lambda_i(\bar{H})^2 n_{\sf pub}}\right\}}{\epsilon^2 n_{\sf priv}^2} \right).\]

For an excess population loss bound, we need to show uniform stability. Note that since the Hessian of $\Psi$, $\bar{H}_{\sf pub}$, is fixed everywhere then PDA-DPMD just applies $\bar{H}_{\sf pub}^{-1} \preccurlyeq O(1/\lambda_{\min}(\bar{H})) \cdot \mathbb{I}_p$ to the noisy gradients. This implies that each step of PDA-DPMD is contractive, and thus that the uniform stability parameter of PDA-DPMD is $O(1/\lambda_{\min}(\bar{H}))$ times that of DP-SGD using the same settings of $\eta_t, T$. The uniform stability of DP-SGD on a convex $L$-Lipschitz loss is $O(\frac{L^2 \sum_t \eta_t}{n})$ (see e.g. Appendix A of \cite{bassily2019private} for a proof). Plugging in the parameters for our setting, this is $O(\log (\epsilon n_{\sf priv}) n_{\sf priv})$, so PDA-DPMD has uniform stability parameter $O(\log (\epsilon n_{\sf priv})/(\lambda_{\min}(\bar{H}) n_{\sf priv}))$. The expected excess population loss is at most the uniform stability parameter plus the expected excess empirical loss, giving the theorem.
\end{proof}

\mypar{Local Stability Properties}
Since in linear regression the public loss function has the same Hessian $\hat{H}_{\sf pub}$ everywhere, mirror descent effectively is DP-SGD, but applying $\hat{H}_{\sf pub}^{-1}$ to the noisy gradient. This allows us to readily characterize the effective noise being added, and show that the noise causes each iterate $\theta_t$ to be moved by an amount proportional to $1/\lambda_\bfv$ in a direction where the strong convexity parameter is $\lambda_\bfv$:

\begin{thm}\label{thm:locstab-linreg}
Let the Hessian of $\Psi$ be $\hat{H}_{\sf pub} = \sum_i \lambda_i \bfv_i \bfv_i^\top$, where $\bfv_i$ are the unit eigenvectors of $\hat{H}_{\sf pub}$. Fix an iteration $t$ as well as starting point $\theta_t$ and private gradient $\bfg_t$ in PDA-DPMD. Let $\bar{\theta}$ be the value of $\theta_{t+1}$ after performing the mirror descent update with $\bfb_t = \boldzero$ at iteration $t$, and let 
where $\hat{\theta}$ be the value of the next iterate $\theta_{t+1}$ if noise is added. Then for any (unit) direction $\bfv = \sum\limits_i a_i \bfv_i$,

{\small \[\Ex{|\langle \hat{\theta} - \bar{\theta}, \bfv \rangle|} = \eta \sigma  \sqrt{\frac{2}{\pi} \cdot \sum\limits_i \left(\frac{a_i}{\lambda_i}\right)^2}.\]}
\end{thm}

In contrast, for DP-SGD, $\Ex{|\langle \hat{\theta} - \bar{\theta}, \bfv \rangle|} = \eta \sigma  \sqrt{\frac{2}{\pi}}$ for all $\bfv$. 

\begin{proof}
Let $\bfb_t$ be the noise added for privacy. Without noise, mirror descent would set $\theta^*$ to be such that:

\[-\eta \bfg_t = \nabla \Psi(\theta^*) - \nabla \Psi(\theta_t).\]

Similarly, given the noisy gradient $\bfg_t + \bfb_t$, mirror descent would set $\hat{\theta}$ to be such that:
\[-\eta (\bfg_t + \bfb_t) = \nabla \Psi(\hat{\theta}) - \nabla \Psi(\theta_t).\]

We then have:

\[ -\eta \bfb_t = \nabla \Psi(\hat{\theta}) - \nabla \Psi(\theta^*).\]

In turn, recalling that $\Psi$ has a fixed Hessian we have:

\[ \hat{\theta} - \theta^* = - \eta \bar{H}_{\sf pub}^{-1} \bfb_t \]

We can now directly prove the theorem:

\[\Ex{|\langle \hat{\theta} - \theta^*, \bfv \rangle|
} = \eta \Ex{ |\langle\bar{H}_{\sf pub}^{-1} \bfb_t, \bfv \rangle|}\]
\[= \eta \Ex{ |\langle(\sum_i \frac{1}{\lambda_i} \bfv_i \bfv_i^\top) \bfb_t, \bfv\rangle| } = \eta \Ex{ |\sum_i \frac{a_i}{\lambda_i}\langle \bfb_t, \bfv_i \rangle |}\]
\[= \eta \Ex{| \sum_i N(0, (a_i / \lambda_i)^2)|} = \eta \Ex{|  N(0, \sum_i (a_i / \lambda_i)^2)|} =\sqrt{\frac{2}{\pi}} \cdot \eta \sigma \sqrt{\sum_i \left(\frac{a_i}{\lambda_i}\right)^2}.\]
\end{proof}

\mypar{Simulation Results} To corroborate our theoretical results with empirical validation, we run PDA-DPMD on synthetic data for the linear regression problem with mean squared error loss.
We vary the dimensionality of the problem $p$ from 500 to 6000. For each $p$, we generate 10,000 private samples and $1.5p$ public samples.
The optimal $\theta^*$ is sampled from $\calN(0,\ind)$. 
To introduce a non-isotropic geometry, we sample the feature vector $\bfx_i$ such that 40 of the first $p/5$ features and 80 of the last $4p/5$ features, chosen uniformly at random, are set to $0.05$, and the rest of the features are set to 0. 
In this way, the expected $\ell_2$-norm of each feature vector (and in turn each gradient) is $O(1)$, and thus the effects of clipping should not vary with $p$. 
The predicted variable $y_i$ is sampled from $\calN(\theta^* \cdot \bfx_i, 0.01)$ so that the population mean squared error loss is always 0.01, i.e. independent of dimension.
We set $\epsilon = 1$, $\delta = 10^{-5}$. 

We consider three algorithms: (i) DP-SGD with a ``cold start'', i.e. using a random initialization, (ii) DP-SGD with a ``warm start'' on the model pre-trained with public data, and (iii) PDA-DPMD after pre-training on public data. 
We perform a grid search over the learning rate, clipping norm, and number of epochs used and report the best empirical loss. We perform 20 trials for each algorithm and dimension.

Note that the exact optimum on the public data can be computed exactly as $\theta_{\sf pub}^* = (\bfX^\top \bfX)^{-1} \bfX^\top \bfy$. 
The mirror descent step can also be solved exactly by applying the inverse of the Hessian $\bfX^\top \bfX$ to the gradient, since the Hessian is the same everywhere. 
For numerical stability, we add a small constant times the identity matrix to the Hessian before computing its inverse. 
We also normalize the Hessian of the loss function so its inverse (which is applied to the gradient before taking a step in PDA-DPMD) has maximum eigenvalue of one. This ensures that if the Hessian were a multiple of the identity matrix, DP-SGD and PDA-DPMD would behave exactly the same for the same hyperparameter choice.

Figure~\ref{fig:csvsws} shows the empirical loss of cold- and warm-start DP-SGD.
Our results show that pre-training with a number of public samples linear in the dimension allows DP-SGD to achieve nearly dimension-independent error. 
Figure~\ref{fig:wsvsmd} compares warm-start DP-SGD and PDA-DPMD. The loss of PDA-DPMD is never worse than that of warm-start DP-SGD, and can be substantially lower for smaller dimensions. 
We observed that the ratio of the maximum and minimum eigenvalues of the Hessian $\bfX^\top \bfX$ decreases as $p$ increases, which means that the Hessian has poorly-concentrated eigenvalues at small $p$ but gets closer to the identity matrix as $p$ increases. 
Since PDA-DPMD recovers warm start DP-SGD when the Hessian is the identity, we can expect that PDA-DPMD obtains less of an advantage over DP-SGD as the Hessian gets closer to the identity. 

\section{Empirical Evaluation}
\label{sec:empEval}

\begin{figure}
\centering
\begin{minipage}{.49\textwidth}
\subcaptionbox{Cold start DP-SGD vs. warm start DP-SGD.\label{fig:csvsws}}
{\includegraphics[width=.49\textwidth]{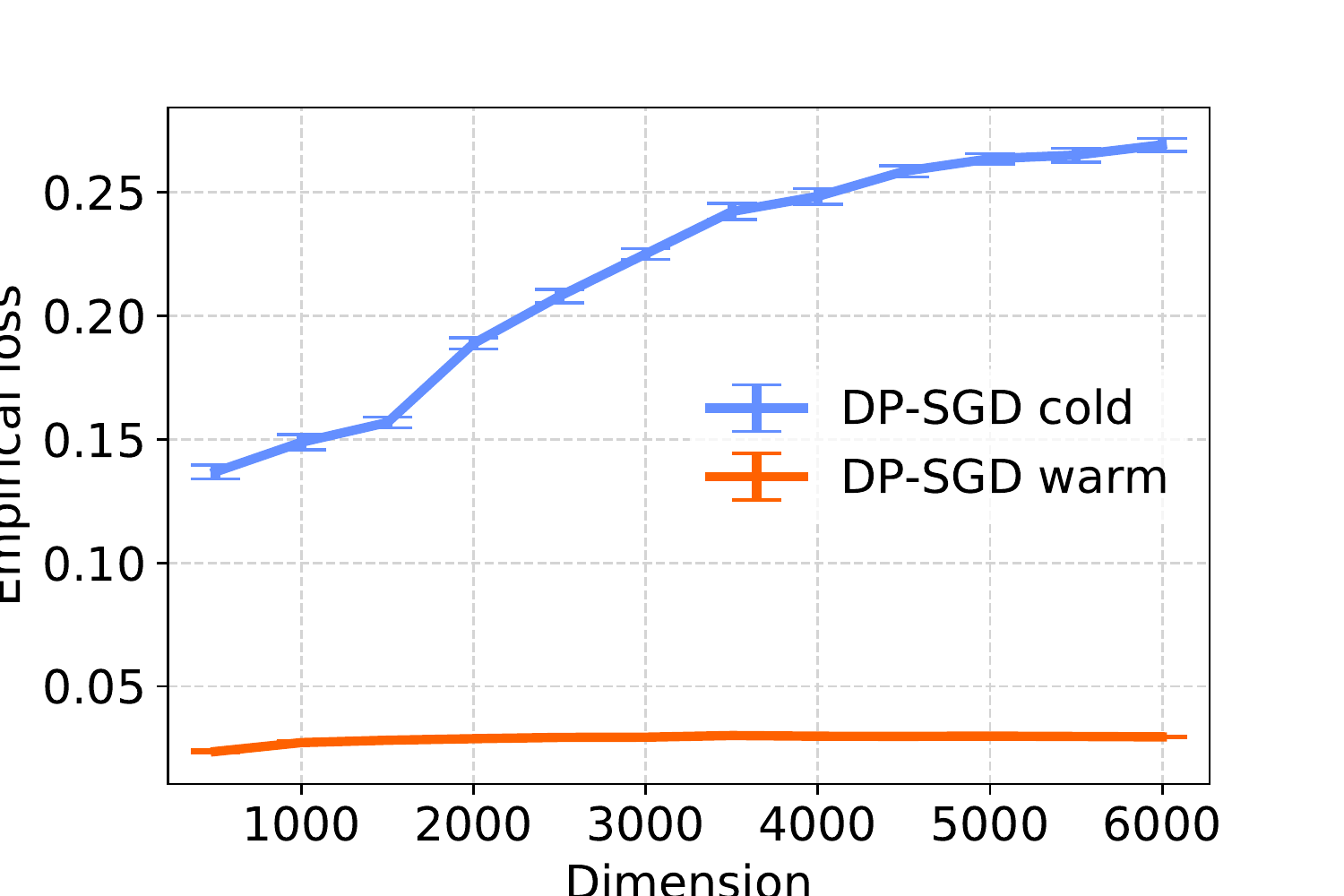}}
\subcaptionbox{Warm start DP-SGD vs. PDA-DPMD.\label{fig:wsvsmd}}
{\includegraphics[width=.49\textwidth]{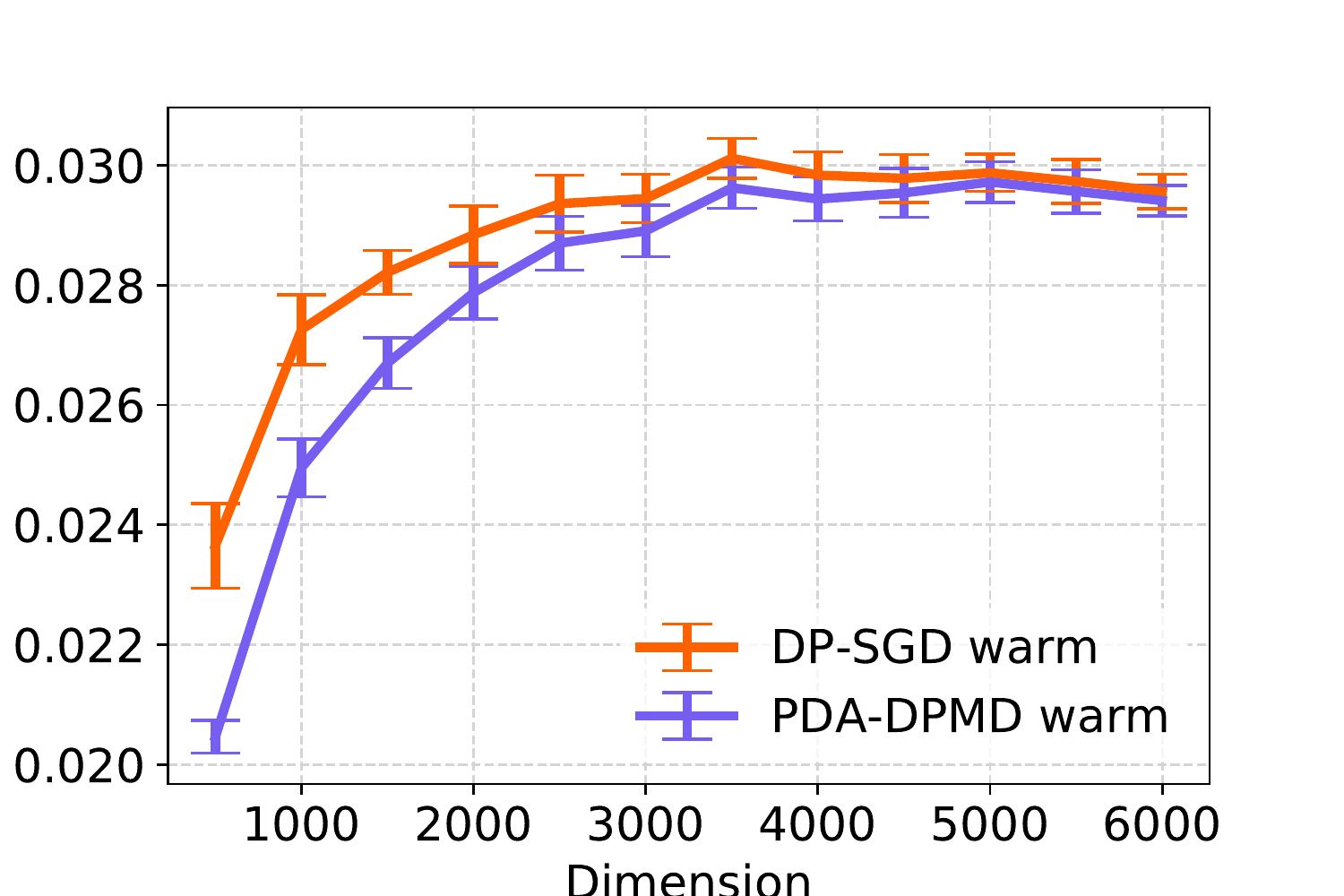}}
\caption{The empirical loss on synthetic linear regression data. 
The mean and error bars for a 95\% confidence interval over 20 
runs are plotted.
The optimal population loss is $0.01$.}
\label{fig:syntheticlinreg}
\end{minipage}
\end{figure}

\subsection{First-order Approximation to Mirror Descent} 
In practice, the Mirror Descent (MD) step in Line~\ref{line:mdupdate} of Algorithm~\ref{alg:dpmd} can be computationally expensive. 
For settings where (i) the problem is \emph{unconstrained}, i.e., $\calC=\mathbb{R}^p$ and (ii) the public loss function $\Psi(\theta)$ may not be strongly convex with respect to the $\ell_2$-norm, we can instead use the following more efficient approximation:
\begin{equation}
\theta_{t+1}\leftarrow \theta_t - \eta_t \left(\alpha_t (\bfg_t + \bfb_t) + (1-\alpha_t) \nabla \Psi(\theta_t) \right),
\label{eq:simpmd}
\end{equation}
where $\eta_t$ is the learning rate, and $\alpha_t \in [0, 1]$ balances the weight of private and public gradient.
The derivation of this formula is in Appendix~\ref{app:first order approximation}.
Notice that $\alpha_t=1$ corresponds to DP-SGD on private data only. 
In our experiment, we decrease $\alpha_t$ with a cosine schedule, i.e. $\alpha_t = \cos(\pi t/ (2 \alphaT))$ where $\alphaT$ is a hyperparameter that controls how fast $\alpha_t$ decays. 
In practice, instead of computing $\bfg_t$ and $\nabla \Psi(\theta_t)$ using all the private and public data, we can estimate them with stochastic gradients. 

\subsection{Evaluation with User-level Differential Privacy}
\label{sec:empEvalFederated}
Now, we evaluate our technique (Algorithm~\ref{alg:dpmd}) with the update step in \eqref{eq:simpmd} on settings suited for user-level DP guarantees.
In particular, we conduct our evaluation for next word prediction on the benchmark StackOverflow~\cite{so_data} dataset, which consists of 342.4k users having 135.8M examples in the training set, and 204.1k users having 16.6M examples in the test set.
Since StackOverflow is naturally
keyed by users, we assume training in a federated learning~\cite{FL1} setting. 
For private users, we limit each user to have at most 256 examples.
For each user having more than 256 examples in StackOverflow, we create a private user with 256 examples and a new user with at most 256 of the overflowing examples. We randomly assign 100 such newly created users (0.03\% of the total number of users) as public users. 

We first pre-train on the public users, then use Algorithm~\ref{alg:dpmd} with the update step in~\eqref{eq:simpmd} (with the updates modeled on the DP Federated Averaging optimizer~\cite{dplangmodels} instead of DP-SGD). The privacy guarantee is thus user-level.
We compare  with two baselines: ``cold start" DP-FedAvg, which uses the private data only, and ``warm start" DP-FedAvg, which pre-trains the model with public data and then fine-tunes with the private data using DP-FedAvg. 
We use the one layer LSTM from~\cite{kairouz2021practical} for our experiments on StackOverflow. See Appendix~\ref{app:exp archs} for more details.

For pre-processing, we use the 10k most frequently used words, and represent all other words with an unknown token. We set the maximum length of any sequence to 20, and each client has a maximum of 256 sentences. 
In each training round, we sample 100 clients, and each client uses a batch size of 16 for local training. We train for 1.6k rounds. See Appendix~\ref{app:exp archs} for more details. We evaluate utility using accuracy on in-vocabulary words and perplexity.

\begin{figure}[h!]
\centering
\begin{subfigure}[b]{0.49\textwidth}
\centering
\includegraphics[width=\textwidth]{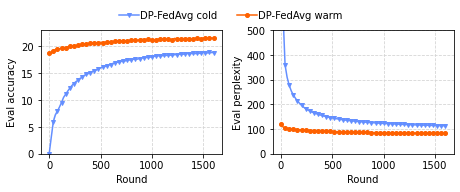}
\caption{DP-FedAvg cold start compared to DP-FedAvg warm start.}
\end{subfigure}
\hfill
\begin{subfigure}[b]{0.49\textwidth}
\centering
\includegraphics[width=\textwidth]{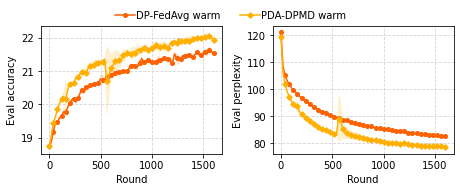}
\caption{DP-FedAvg warm start and PDA-DPMD warm start.}
\end{subfigure}
\caption{Stackoverflow. 0.03\% public data. Eval accuracy / perplexity vs. training rounds. Averaged over 3 runs.}
\label{fig:stackoverflow}
\end{figure}

\begin{table}[!htb]
\centering
\begin{tabular}{l l l}
\toprule
Algorithm& Accuracy & Perplexity \\
\midrule
DP-FedAvg cold   &18.52 $\pm$ 0.04     & 116.21 $\pm$ 0.08 \\
DP-FedAvg warm   &21.2 $\pm$ 0.03       & 85.03 $\pm$ 0.06 \\
PDA-DPMD warm     & \textbf{21.77 $\pm$ 0.02}       & \textbf{80.51 $\pm$ 0.2} \\
\bottomrule
\end{tabular}
\caption{Mean and standard deviation of the Test accuracy and perplexity for trained models on StackOverflow.}
\label{table:stackoverflow}
\end{table}

We compare warm start PDA-DPMD to warm start DP-FedAvg. We fix the noise multiplier to $\sigma = 0.4$ so that the client-level $\varepsilon= 8.32$ for $\delta = 10^{-6}$. We perform a grid search over the server learning rate $\{0.5, 1.0, 3.0\}$, the client learning rate $\{0.1, 0.2, 0.5\}$, and clipping norm $\{0.3, 1.0\}$ for both methods. We perform an additional search for the optimal decay schedule for $\alpha$ for PDA-DPMD using $\cos(\pi t / (iT))$ for $i \in \{2, 3, 4, 5, 8\}$, where $T = 1600$.  
In Figure~\ref{fig:stackoverflow}, we plot the eval accuracy / perplexity across training for $\alpha_t = cos(\pi t / (5T))$. 
In Table~\ref{table:stackoverflow}, we present the test accuracy and perplexity for the final trained models. We see that PDA-DPMD obtains a 0.57\% absolute (2.7\% relative) increase in accuracy, and a 5.3\% drop in perplexity compared to warm start DP-FedAvg.
\emph{We demonstrate an increased benefit from incorporating the public data in training over and above just pre-training}, which to our knowledge, has not been achieved in prior work.

\mypar{Remark} A follow-up work~\cite{li2022private}  uses in-distribution public data for improving the privacy-utility trade-offs, by using it to incorporate adaptivity in DP optimizers. 
Their empirical results for  StackOverflow are incomparable to ours because they consider a subsampled version of StackOverflow with a total of 400 users, whereas we conduct our experiments on the complete StackOverflow dataset with $342.4K$ users.

\subsection{Evaluation with Sample-level Differential Privacy}
\label{sec:empEvalCentral}
Next, we demonstrate the efficacy of our technique with the update step in~\eqref{eq:simpmd} on two real world tasks across three benchmark datasets: next word prediction on WikiText-2~\citep{WikiText}, and image classification on CIFAR-10~\citep{cifar10} and EMNIST (ByMerge split)~\citep{cohen_afshar_tapson_schaik_2017}. 
For each dataset, we randomly assign 4\% of the original training data as public, and the rest as private. We do not consider a larger amount of in-distribution public data as that could make the problem trivial. 

As in Section~\ref{sec:empEvalFederated}, we first pre-train on the public data, then use Algorithm~\ref{alg:dpmd} with update rule ~\eqref{eq:simpmd}.
We compare our algorithm with two baselines (analogous to those in  Section~\ref{sec:empEvalFederated}): ``cold start" DP-SGD, and ``warm start" DP-SGD.
In this setting too, we see benefits from the public data above just pre-training.
For WikiText-2, we use an LSTM model from~\cite{asi2021private}.
For CIFAR-10 and EMNIST, we use network architectures considered in prior works~\citep{papernot2020tempered,kairouz2021practical}.  See Appendix~\ref{app:exp archs} for more details.

\mypar{Empirical Evaluation on WikiText-2}
Our setup mainly follows~\citet{asi2021private}.
As a preprocessing step, we take the top 8k most frequent words and convert the rest into a special token representing unknown word. The dataset is then split into 48.8k length-35 sequences, and we consider sequence-level privacy here.
After pre-training, we fine-tune the model with batch size 250 for 20 epochs. 
We search for optimal $\alphaT$ in $\{100, 200, 500\}$.
For two different privacy levels,   $\epsilon=15.7$ and $1.71$ at $\delta=10^{-5}$ (corresponding to $\sigma=0.5$ and $1.08$, respectively), Figure~\ref{fig:wiki 4} shows the test loss for PDA-DPMD with $\alphaT = 100$ and $50$ respectively, and for the two baselines. 
From Table~\ref{tab:res}, we see that PDA-DPMD obtains the smallest test loss for both the privacy levels.
Also, comparing the two DP-SGD baselines, we can see that using public data for pre-training provide trained models with higher utility.

\begin{figure}
    \centering
    \begin{minipage}{0.49\textwidth}
\subcaptionbox{$\sigma=0.5$. $\epsilon=15.7$.}
{\includegraphics[width=.45\textwidth]{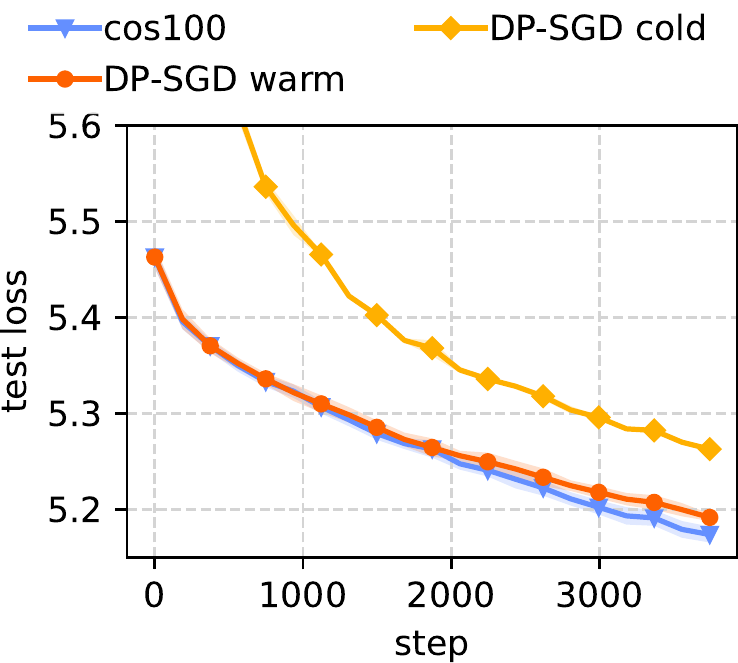}}
\hfill
\subcaptionbox{$\sigma=1.08$. $\epsilon=1.71$.}
{\includegraphics[width=.45\textwidth]{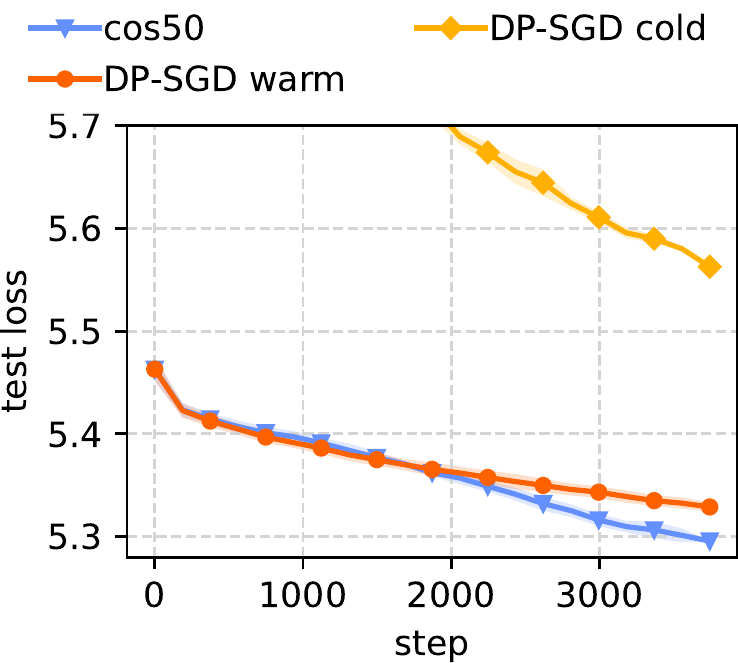}}
\caption{WikiText-2. 4\% public data. Validation and test loss. Averaged over $3$ runs.
Full plots in Appendix~\ref{app:full plot}.}
\label{fig:wiki 4}
\end{minipage}
\end{figure}

\begin{table}
\scriptsize
\centering
\caption{Metrics for the final models for each configuration.}
\label{tab:res}
\begin{tabular}{c l l l}
\toprule
Dataset, metrics& Algorithm&Smaller $\sigma$ &Larger $\sigma$ \\
\midrule
\multirow{3}{*}{\makecell{WikiText-2,\\test loss}}
&DP-SGD cold    &5.2626    &5.5627 \\
&DP-SGD warm    &5.1914    &5.3288 \\
&PDA-DPMD       &\textbf{5.1736}    &\textbf{5.2956} \\
\midrule
\multirow{3}{*}{\makecell{CIFAR-10,\\accuracy / test loss }}
&DP-SGD cold    &62.9633 / 1.4225   &40.6000 / 1.6890 \\
&DP-SGD warm    &66.3933 / 1.2371   &53.4100 / 1.3462 \\
&PDA-DPMD       &\textbf{67.0300 / 1.1435}   &\textbf{55.3950 / 1.2785} \\
\midrule
\multirow{3}{*}{\makecell{EMNIST,\\accuracy / test loss}}
&DP-SGD cold    &87.5671 / 0.5422   &84.7270 / 0.6170 \\
&DP-SGD warm    &87.8534 / 0.5089   &86.3352 / 0.5586 \\
&PDA-DPMD       &\textbf{87.9860 / 0.4706}   &\textbf{86.7229 / 0.4982} \\
\bottomrule
\end{tabular}
\end{table}

Though our work's focus is on in-distribution public data, we additionally compare with SoTA~\citep{asi2021private} which uses WikiText-103~\citep{WikiText} as the public data.
In that setting\footnote{The implementation of \citet{asi2021private} was not made public at the time of this submission, even after contacting the authors.
 We make our best effort to match their experiment setup. Since the algorithms can be sensitive to hyperparameter choices, for a fair comparison we only use  their quoted results.},
our warm start DP-SGD baseline is better than the proposed SoTA in~\citep{asi2021private} by 1.1\% for $\epsilon=1.0$, and 6.6\% for $\epsilon=3.0$ in terms of test perplexity.

\mypar{Empirical Evaluation on CIFAR-10}
CIFAR-10 consists of 50k training images and 10k test images from 10 classes.
After pre-training, we fine-tune the model with batch size 500 for 100 epochs.
We search for optimal $\alphaT$ in $\{200, 500, 1000, 2000, 5000\}$.
In Figure~\ref{fig:cifar10 4}, for two different privacy levels, $\epsilon=3.51$ and $0.19$ at $\delta=10^{-5}$ (corresponding to $\sigma=1.51$ and $20.0$, respectively), we report the test loss and accuracy for $\alphaT = 2000$, and for the two baselines.
From Table~\ref{tab:res}, we see that PDA-DPMD provides the best accuracy (even if by a small margin over the warm started DP-SGD baseline). Moreover, PDA-DPMD also results in significantly lower test loss compared to both the baselines for both privacy levels.

\begin{figure}[h!]
\centering
\begin{subfigure}[b]{0.49\textwidth}
\centering
\includegraphics[width=\textwidth]{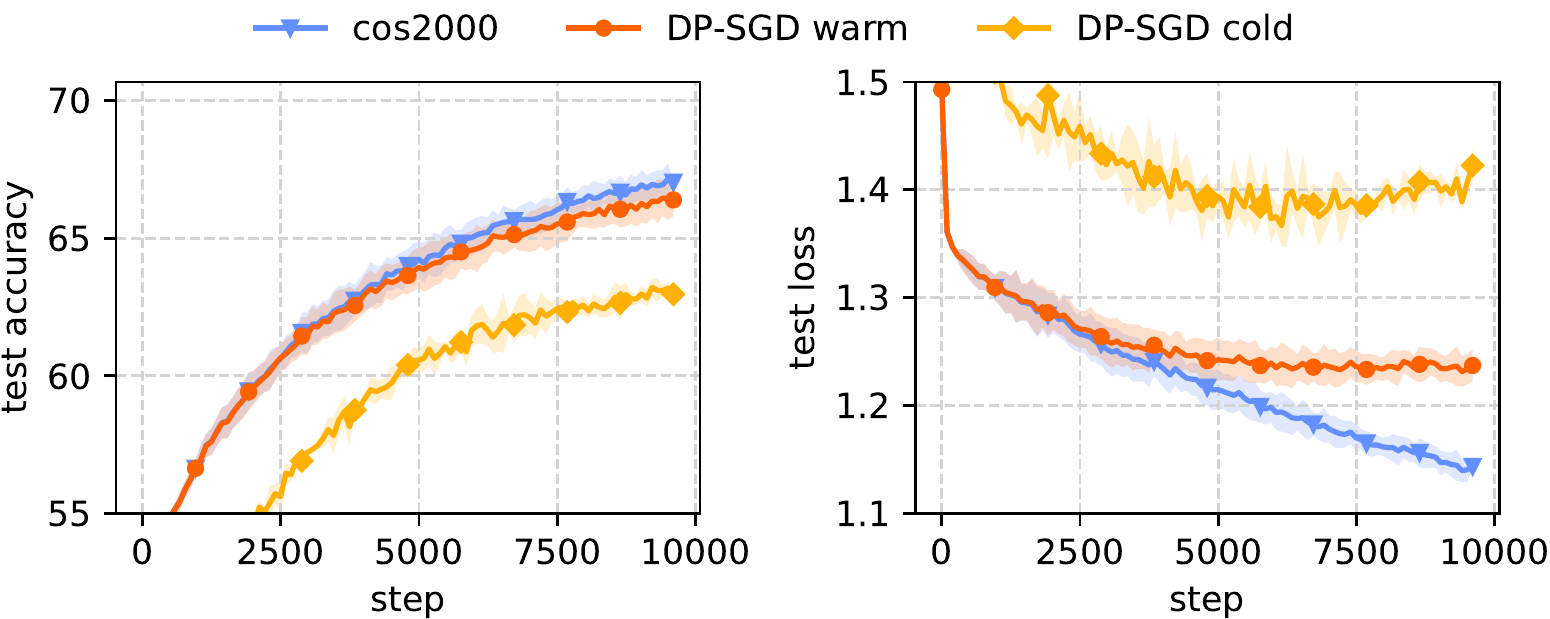}
\caption{$\sigma=1.51$. $\epsilon=3.51$.}
\end{subfigure}
\hfill
\begin{subfigure}[b]{0.49\textwidth}
\centering
\includegraphics[width=\textwidth]{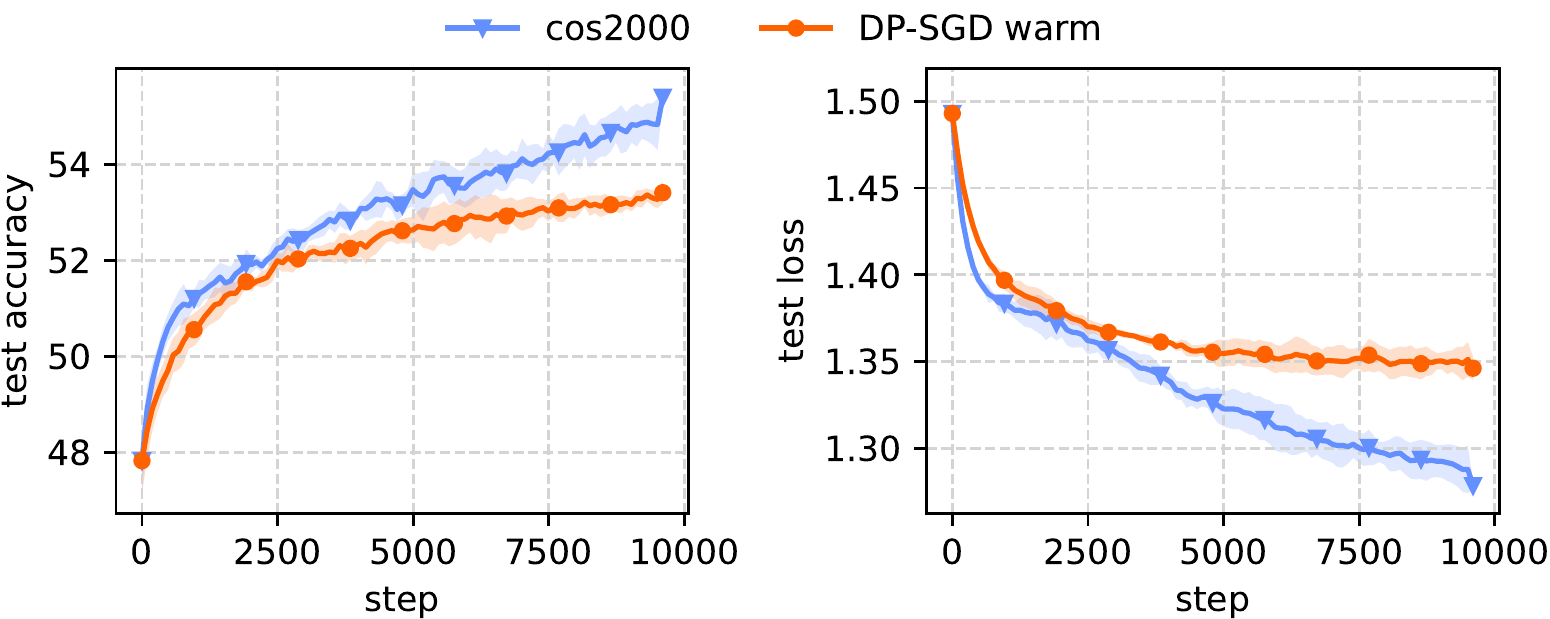}
\caption{$\sigma=20.0$. $\epsilon=0.19$. Full plot in Appendix~\ref{app:full plot}.}
\end{subfigure}
\caption{CIFAR-10. 4\% public data. Test accuracy / loss vs. training steps. Averaged over $3$ runs.}
\label{fig:cifar10 4}
\end{figure}

\mypar{Empirical Evaluation on EMNIST}
EMNIST (ByMerge split) consists of 697.9K training images and 116.3k test images from 47 classes.
After pre-training, we fine-tune
with batch size 500 for 50 epochs. 
We search for optimal $\alphaT$ in $\{200, 500, 1000, 2000, 5000\}$.
In Figure~\ref{fig:emnist 4} and Table~\ref{tab:res}, for $\sigma=0.41$ and $1.89$, corresponding to privacy $\epsilon=25.80$ and $0.48$ at $\delta=10^{-6}$, we report the test loss and accuracy for $\alphaT = 500$, and for the two baselines. 
We see a similar trend as with CIFAR-10.

\begin{figure}[h!]
\centering
\begin{subfigure}[b]{0.49\textwidth}
\centering
\includegraphics[width=\textwidth]{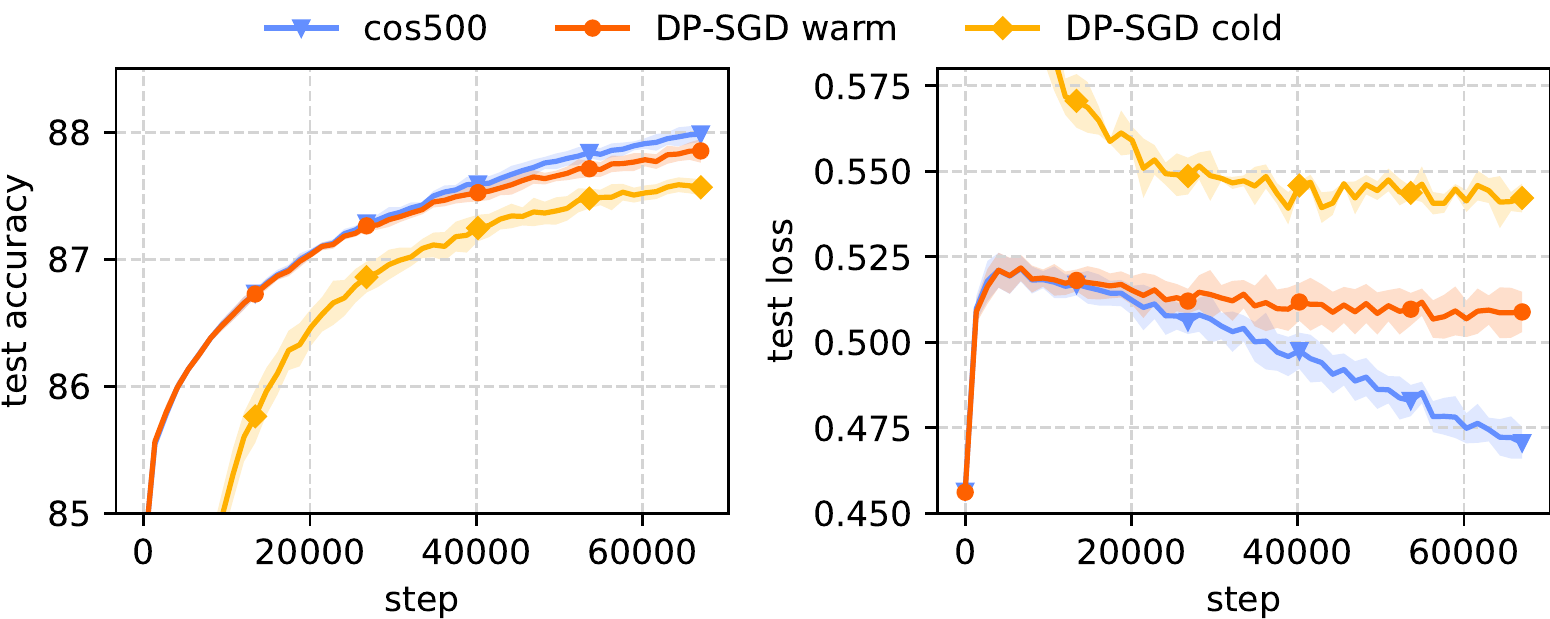}
\caption{$\sigma=0.41$. $\epsilon=25.80$.}
\end{subfigure}
\hfill
\begin{subfigure}[b]{0.49\textwidth}
\centering
\includegraphics[width=\textwidth]{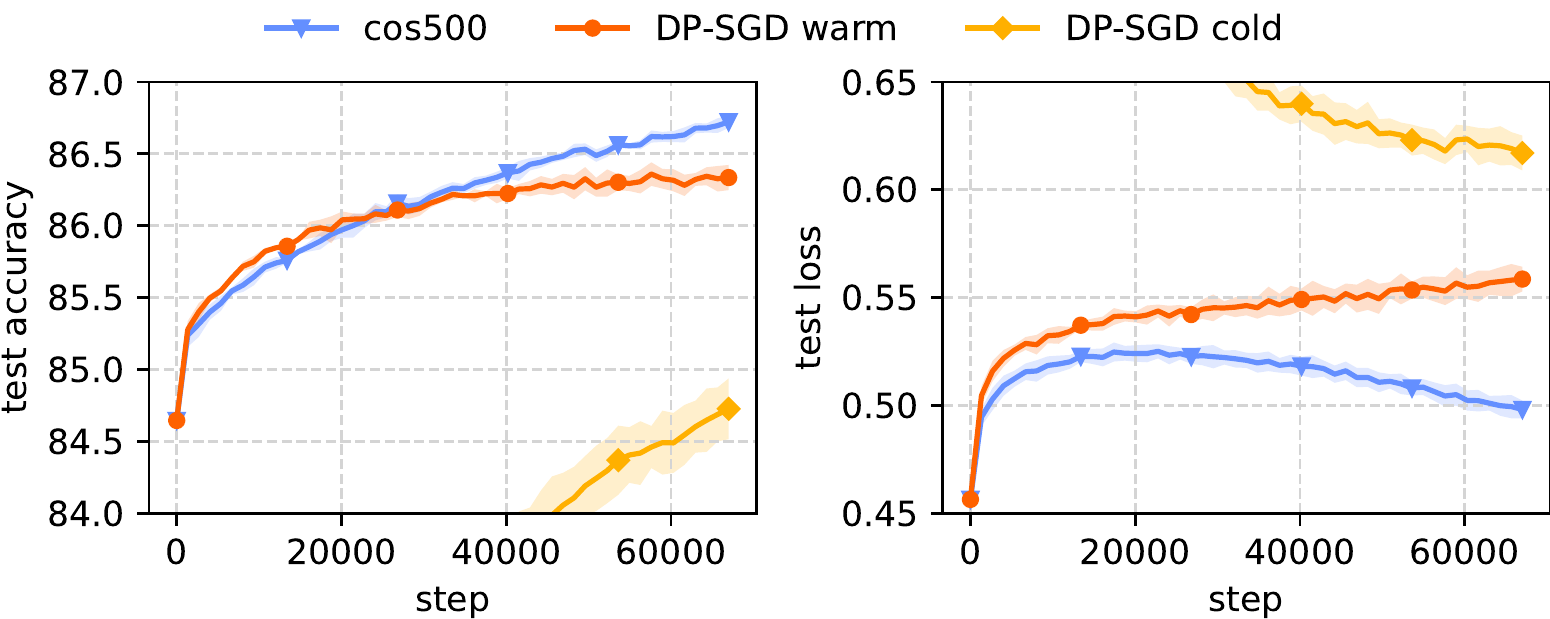}
\caption{$\sigma=1.89$. $\epsilon=0.48$.}
\end{subfigure}
\caption{EMNIST. 4\% public data. Test accuracy / loss vs. training steps. Averaged over $3$ runs.}
\label{fig:emnist 4}
\end{figure}
\bibliographystyle{icml2022}
\bibliography{reference}
\appendix
\section{Notation Reference}\label{sec:notation}

We recall here the notation commonly used used throughout the paper:

\begin{table}[h]
\begin{tabular}{|c|c|}
\hline
$\alpha$ & A weighting of public and private gradients \\
\hline
$\bfb$ & Noise added to gradients for privacy \\
\hline
$B_\Psi$ & The Bregman divergence induced by $\Psi$ \\
\hline
$\mathcal{C}$ & The constraint set \\
\hline
$d$ & A single sample \\
\hline
$D_{\sf pub}, D_{\sf priv}$ & The public and private data sets respectively \\
\hline
$\mathcal{D}$ & The universe of samples \\
\hline
$\epsilon, \delta$ & The privacy parameters \\
\hline
$\eta$ & The learning rate/step size \\
\hline
$\bfg$ & A (batch) gradient \\
\hline
$G_Q$ & The Gaussian width of $Q$ \\
\hline
$\ell_{\sf pub}, \ell_{\sf priv}$ & The (per-example) public and private loss functions respectively \\
\hline
$L$ & The Lipschitz constant of $\ell_{\sf priv}$ \\
\hline
$\mathcal{L}$ & A loss on dataset $D$, or the population loss\\
\hline
$\lambda$ & An eigenvalue\\
\hline
$n_{\sf pub}, n_{\sf priv}$ & The number of public and private samples respectively \\
\hline
$p$ & The dimensionality of the optimization problem\\ 
\hline
$\Psi$ &  The empirical public loss, i.e. $\Psi(\theta) = \frac{1}{|D_{\sf pub}|} \sum_{d \in D_{\sf pub}} \ell(\theta; d)$; also the mirror map\\ 
\hline
$T$ & The number of iterations in the optimization algorithm\\ 
\hline
$\tau$ & The distribution from which the training samples are drawn\\ 
\hline
$\theta$ & A solution to the optimization problem\\ 
\hline
\end{tabular}
\end{table}
\section{Additional Details on Real-world Experiments}
\label{app:exp real central}

\subsection{First-order Approximation to Mirror Descent}
\label{app:first order approximation}

In practice, the Mirror Descent (MD) step in Line~\ref{line:mdupdate} of Algorithm~\ref{alg:dpmd} is computationally the most challenging. In the following, we provide an approximation of this step in the setting where i) the problem is \emph{unconstrained}, i.e., $\calC=\mathbb{R}^p$ and ii). the public loss $\Psi(\theta)$ may not be strongly convex with respect to the $\ell_2$-norm. This approximation makes Algorithm~\ref{alg:dpmd} efficient in practice. 

Consider the following equivalence of Line~\ref{line:mdupdate} of Algorithm~\ref{alg:dpmd}, with $\Psi(\theta)$, the public loss on $\dpub$, replaced by $\hPsi(\theta)=\Psi(\theta)+\frac{1}{2}\ltwo{\theta}^2$.  This follows from Lemma 5.5 of~\cite{hazan2019introduction}.
\begin{align}
\theta_{t+1}&\leftarrow \argmin\limits_{\theta\in\mathbb{R}^p} \ip*{\sum\limits_{i=1}^t\eta_i \left(\bfg_i+\bfb_i\right)}{\theta}+\hPsi(\theta)\label{eq:mdlazy}\\
&=\argmin\limits_{\theta\in\mathbb{R}^p} \sum\limits_{i=1}^t\eta_i\left(\ip*{\bfg_i+\bfb_i}{\theta}+\frac{1}{t \eta_i}\Psi(\theta)\right)+\frac{1}{2}\ltwo{\theta}^2\nonumber\\
&\approx \argmin\limits_{\theta\in\mathbb{R}^p} \sum\limits_{i=1}^t\eta_i \ip*{\bfg_i+\bfb_i+\frac{1}{t \eta_i}\nabla\Psi(\theta_i)}{\theta}+\frac{1}{2}\ltwo{\theta}^2\label{eq:abd},
\end{align}

where \eqref{eq:abd} follows from the first-order approximation $\Psi(\theta)\approx\Psi(\theta_i)+\ip{\nabla \Psi(\theta_i)}{\theta-\theta_i}$. 

In the experiments, we replace $\bfg_i+\bfb_i+\frac{1}{t\cdot\eta_i}\nabla\Psi(\theta_i)$ with $\alpha_i (\bfg_i+\bfb_i)+(1-\alpha_i)\nabla\Psi(\theta_i)$, where $\alpha_i\in(0,1]$. This reparamertization helps with more effective hyperparameter tuning while training deep learning models. We therefore have the update rule~\eqref{eq:simpmd}.
\subsection{Setup}

\mypar{Network architectures}
\label{app:exp archs}
Table~\ref{table:exp archs} shows model architectures for CIFAR-10,
EMNIST, WikiText-2, \& StackOverflow.
 
\begin{table*}[!htbp]
\caption{Model architectures for real data experiments.}
\label{table:exp archs}
\begin{subtable}[t]{\textwidth}
\caption{Model architecture for CIFAR-10.}
\label{table:cifar10_nn}
\centering
\begin{tabular}{ c c }
\toprule
Layer & Parameters \\
\midrule
Convolution  $\times 2$ & 32 filters of $3 \times 3$, strides 1\\
Max-Pooling & $2 \times 2$, stride 2 \\
Convolution  $\times 2$ & 64 filters of $3 \times 3$, strides 1 \\
Max-Pooling & $2 \times 2$, stride 2 \\
Convolution  $\times 2$ & 128 filters of $3 \times 3$, strides 1 \\
Max-Pooling & $2 \times 2$, stride 2 \\
Fully connected & $128$ units \\
Softmax & - \\
\bottomrule
\end{tabular}
\end{subtable}

\begin{subtable}[t]{0.45\textwidth}
\caption{Model architecture for EMNIST.}
\label{table:mnist_nn}
\centering
\begin{tabular}{ c c }
\toprule
Layer & Parameters \\
\midrule
Convolution & 16 filters of $8 \times 8$, strides 2 \\
Convolution & 32 filters of $4 \times 4$, strides 2 \\
Fully connected & $32$ units \\
Softmax & - \\
\bottomrule
\end{tabular}
\end{subtable}
~
\centering
\begin{subtable}[t]{0.45\textwidth}
\caption{Model architecture for WikiText-2.}
\label{table:stackoverflow_nn}
\centering
\begin{tabular}{ c c }
\toprule
Layer & Parameters \\
\midrule
Input & 8000 \\
Fully connected & 120 \\
LSTM $\times 2$ & 120 hidden units \\
Fully connected & 8000 \\
Softmax & -\\
\bottomrule
\end{tabular}
\end{subtable}
-
\centering
\begin{subtable}[t]{\textwidth}
\caption{Model architecture for StackOverflow.}
\label{table:stackoverflow_nn}
\centering
\begin{tabular}{ c c }
\toprule
Layer & Parameters \\
\midrule
Input & 0 \\
embedding & 960384 \\
LSTM & 2055560 \\
Fully connected & 64416 \\
Fully connected & 970388 \\
Softmax & -\\
\bottomrule
\end{tabular}
\end{subtable}
\end{table*}

\mypar{Hyperparameter Tuning}
We keep the clipping norm to be 1. 
One small difference with the standard DP-SGD update rule is that we enforce an additional clipping step for the privatized gradient for the image classification task, where the clipping norm is the same as the clipping norm of for individual gradient. The reason for this additional step is that the norm of the averaged clipping gradients should still be upper bounded by the clipping norm.

The only hyperparameters that need to be tuned are the learning rate and $\alphaT$ that controls the decaying of $\alpha_t$. For the learning rate, we consider a grid of $\{1, 2, 5\} \times 10^{i}$ for different $i$s such that the optimal learning rate does not appear on the boundary. We search for the optimal $\alphaT$ in $\{100, 200, 500\}$ for WikiText-2 and $\{200, 500, 1000, 2000, 5000\}$ for the image classification tasks.

\subsection{Full Plots for Section~\ref{sec:empEval}}
\label{app:full plot}

In Figure~\ref{fig:wiki 4 full}, we plot the complete version of Figure~\ref{fig:wiki 4},
and in Figure~\ref{fig:cifar10 4 full}, we show the complete version of Figure~\ref{fig:cifar10 4}.

\begin{figure}[htb!]
\centering
\begin{subfigure}[b]{0.49\textwidth}
\centering
\includegraphics[width=\textwidth]{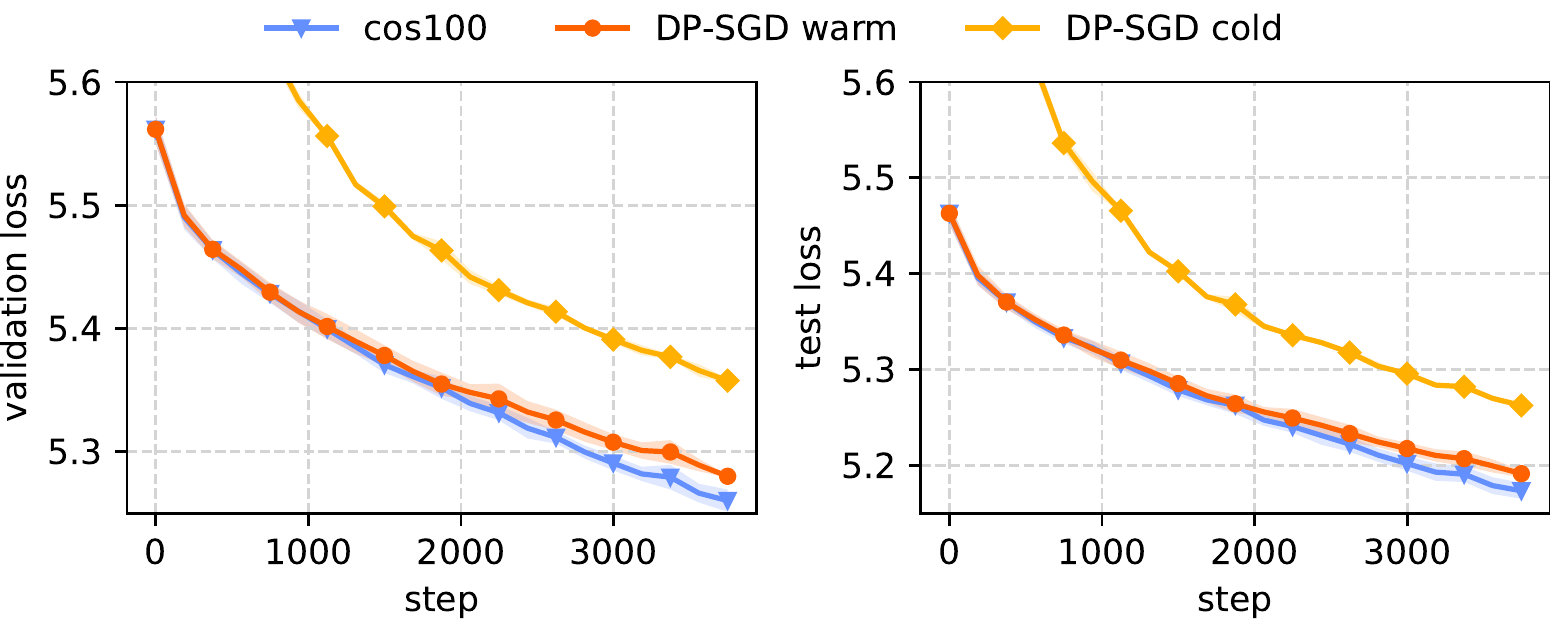}
\caption{$\sigma=0.5$}
\end{subfigure}
\hfill
\begin{subfigure}[b]{0.49\textwidth}
\centering
\includegraphics[width=\textwidth]{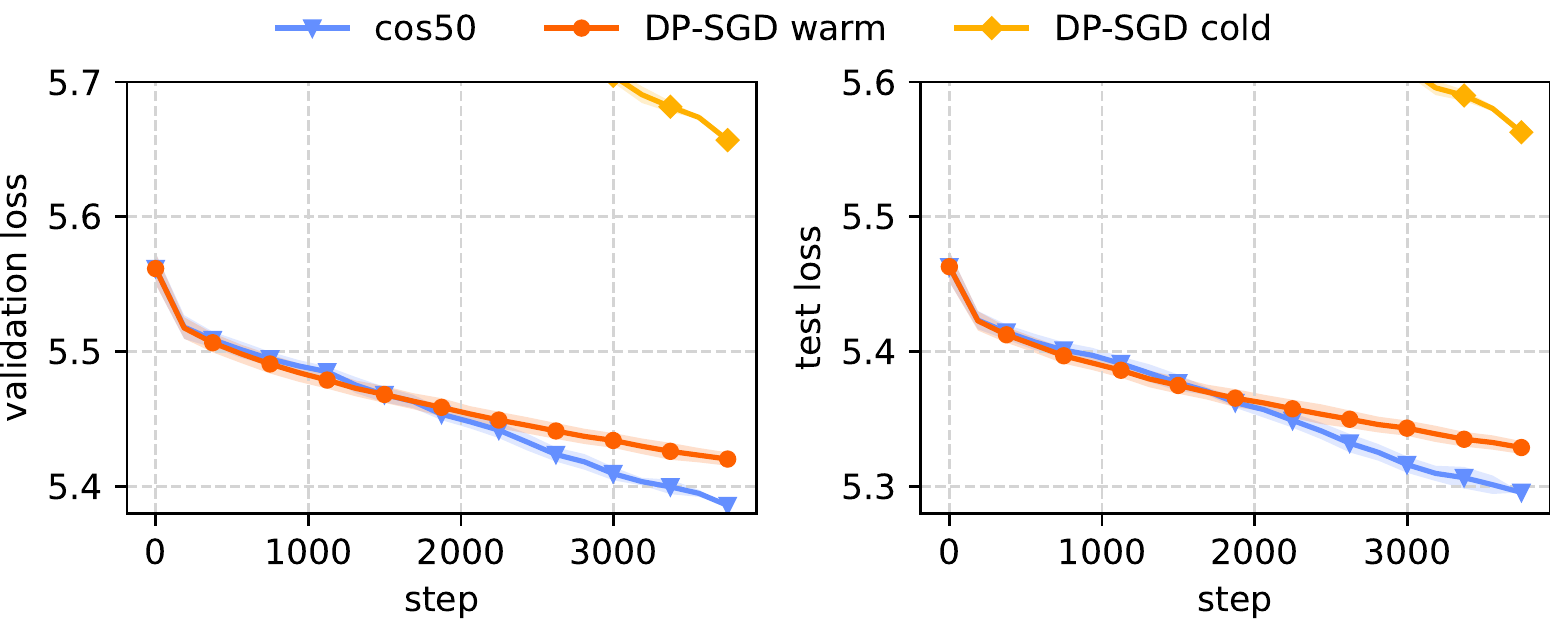}
\caption{$\sigma=1.08$}
\end{subfigure}
\caption{Full plot for Figure~\ref{fig:wiki 4}.
WikiText-2. 4\% public data.}
\label{fig:wiki 4 full}
\end{figure}

\begin{figure}[h!]
\centering
\begin{subfigure}[b]{0.49\textwidth}
\centering
\includegraphics[width=\textwidth]{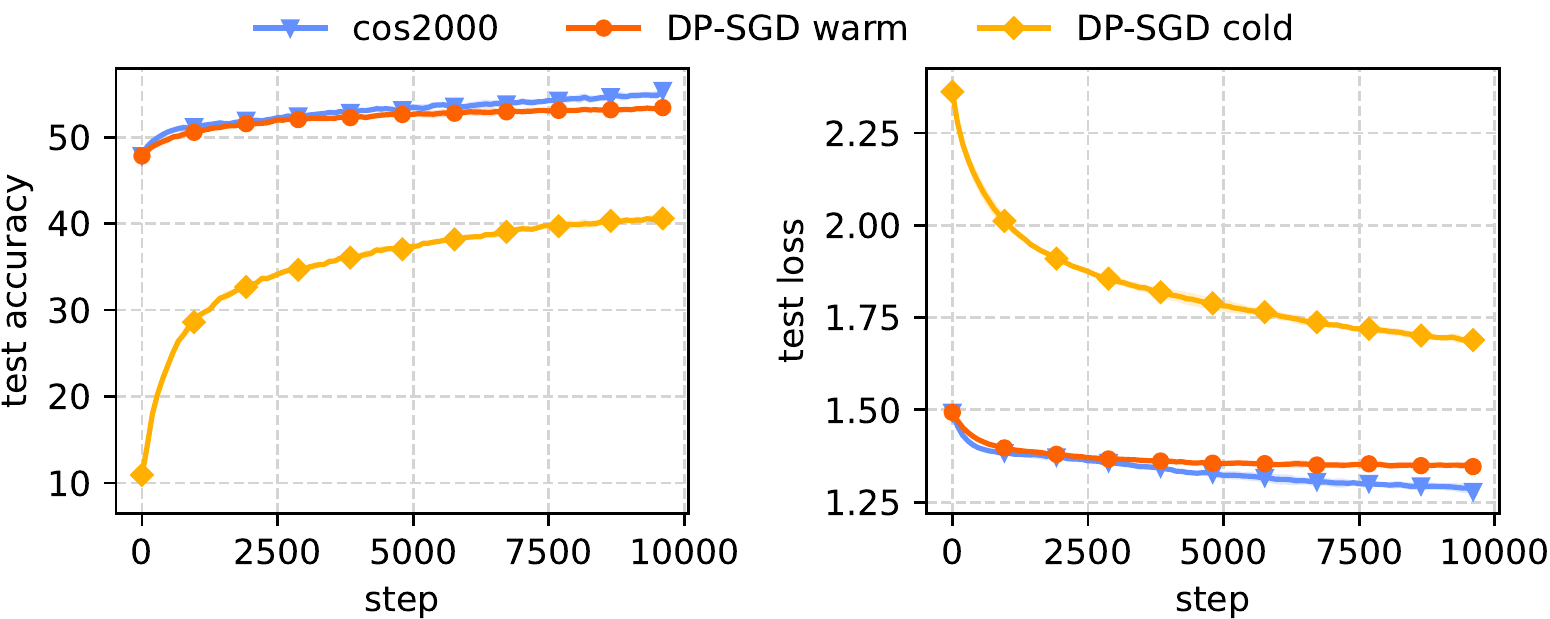}
\caption{$\sigma=20.0$.}
\end{subfigure}
\caption{
Full plot of Figure~\ref{fig:cifar10 4}.
CIFAR-10. 4\% public data. Test accuracy / loss vs. training steps. Averaged over $3$ runs.
}
\label{fig:cifar10 4 full}
\end{figure}

\subsection{WikiText-2 with WikiText-103 as Public Data}
\label{app:wiki2}
We compare with the SoTA~\citep{asi2021private} which uses WikiText-103 as public data. Specifically, we consider their ``LargeAux'' setting under $\epsilon=1.0$ and $3.0$.
Since the implementation for \citet{asi2021private} is not public as of writing this work, we make our best effort to match their experiment setup. We note that the data preprocessing and the number of iterations used (thus the noise multiplier for achieving the same $\epsilon$) might differ.

We preprocess WikiText-103 as follows.
After processing WikiText-2 as described in Section~\ref{sec:empEval}, we convert all words that does not appear in the processed WikiText-2 as the unknown token. Then, we split the sentences into length-35 sequences, and remove all sequences that overlap with WikiText-2.
Finally, we randomly sample 48,764 sequences, in order to match the ``LargeAux'' setting where the public dataset is of the same size as the private training dataset.

Figure~\ref{fig:wiki wiki103 new} shows the results.
In our setting, cold start DP-SGD reaches similar log perplexity as those in~\citep{asi2021private}, while the warm-start DP-SGD is already better than \cite{asi2021private} (LargeAux).
The final test log perplexities are summarized below, with the results in \citep{asi2021private} converted from perplexity to log perplexity. 

\begin{table}[h]
\centering
\begin{tabular}{ l c c }
\toprule
Algorithm       & $\epsilon=3.0$ & $\epsilon=1.0$ \\
\midrule
\cite{asi2021private} DP-SGD (cold)    & 5.4819 & 5.6623 \\
\cite{asi2021private} (LargeAux)       & 5.4324 & 5.5254 \\
\midrule
Our DP-SGD (cold)    & 5.4030 & 5.5956 \\
Our DP-SGD (warm)    & 5.3646 & 5.5141 \\
\bottomrule
\end{tabular}
\end{table}

\begin{figure}[h!]
\centering
\begin{subfigure}[b]{0.49\textwidth}
\centering
\includegraphics[width=\textwidth]{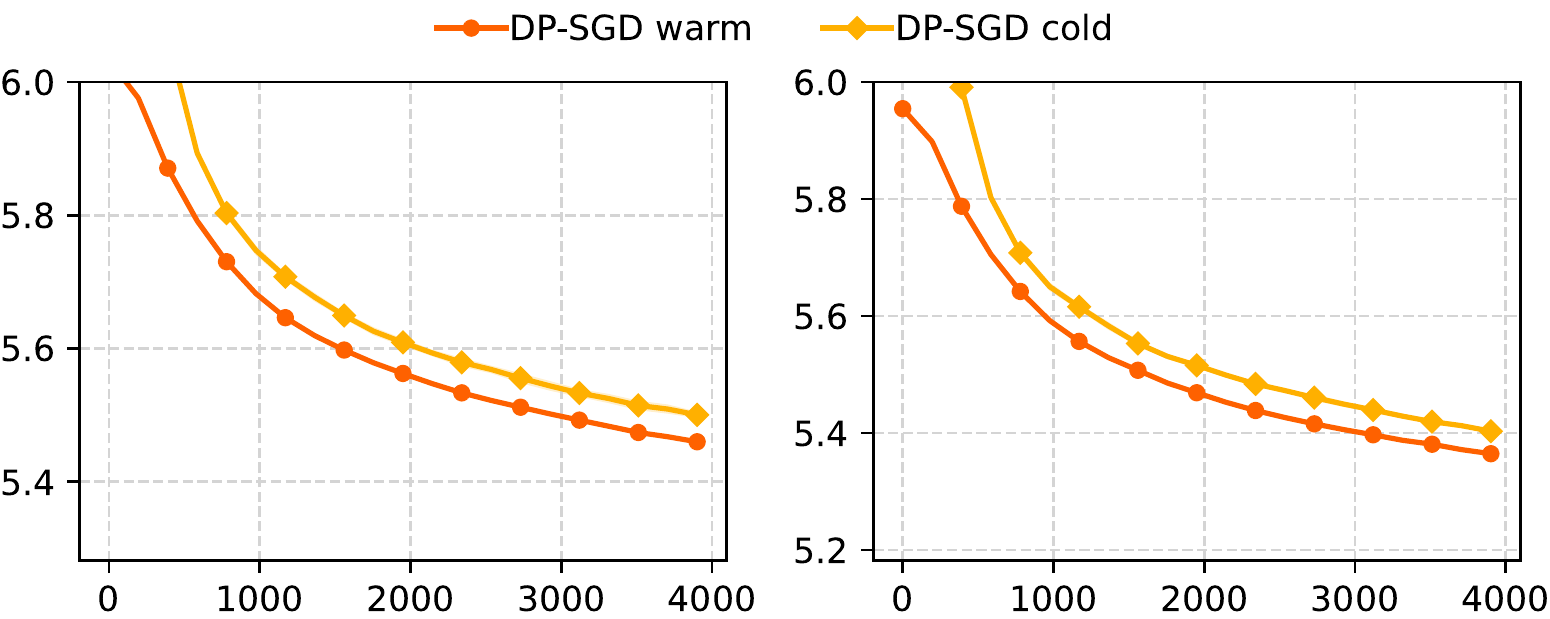}
\caption{$\sigma=0.83$, $\epsilon=3.0$.}
\end{subfigure}
\hfill
\begin{subfigure}[b]{0.49\textwidth}
\centering
\includegraphics[width=\textwidth]{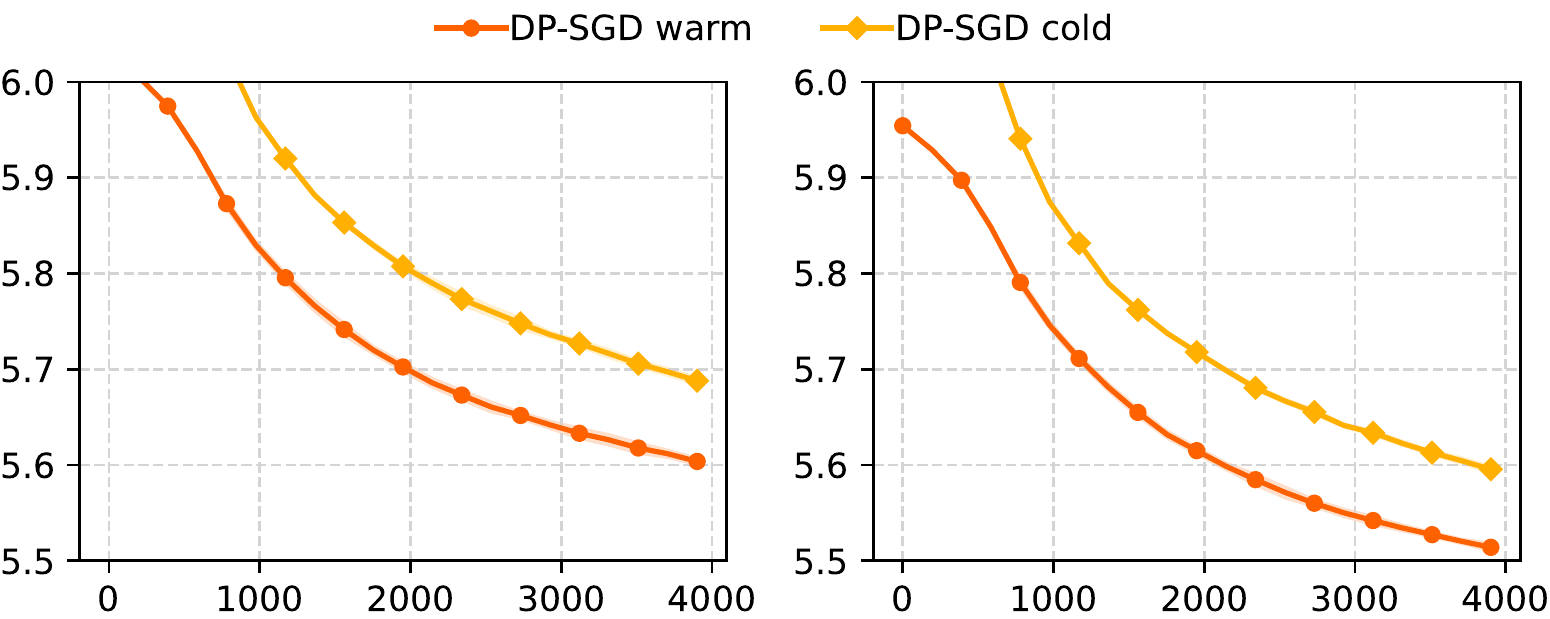}
\caption{$\sigma=1.49$, $\epsilon=1.0$.}
\end{subfigure}
\caption{WikiText-2. WikiText-103 as public data.}
\label{fig:wiki wiki103 new}
\end{figure}
\end{document}